%% file: main.tex
\documentclass{article}

 \usepackage[preprint]{neurips_2026}

\usepackage[utf8]{inputenc} 
\usepackage[T1]{fontenc}    
\usepackage{hyperref}       
\usepackage{url}            
\usepackage{booktabs}       
\usepackage{amsfonts}       
\usepackage{nicefrac}       
\usepackage{microtype}      
\usepackage{xcolor}         
\usepackage{amsmath}
\usepackage{amssymb}
\usepackage{amsthm}
\usepackage{mathtools}
\usepackage{algorithm}
\usepackage{algorithmic}

\newtheorem{proposition}{Proposition}

\usepackage{graphicx}
\usepackage{subcaption}
\usepackage{multirow}
\usepackage{wrapfig}
\usepackage{bbm}
\usepackage{enumitem}
\usepackage{tabularx}

\newcommand{\Qs}{Q^{\mathrm{s}}}
\newcommand{\Vs}{V^{\mathrm{s}}}
\newcommand{\Lrew}{\mathcal{L}_{\text{reward}}}
\newcommand{\Lsafe}{\mathcal{L}_{\text{safety}}}

\usepackage[table]{xcolor}
\usepackage{colortbl}
\definecolor{bestgreen}{RGB}{220,245,220}
\definecolor{badred}{RGB}{255,228,228}
\definecolor{midyellow}{RGB}{255,245,210}
\definecolor{oursblue}{RGB}{225,235,255}

\title{ShieldVLA: Feasibility-Aware Safety Alignment for Vision-Language-Action Models}

\author{%
  Manan Tayal \thanks{ Work Done while at Microsoft Research} \\
  Indian Institute of Science (IISc) \\
  \texttt{manantayal@iisc.ac.in} \\
  \And
  Akshay Nambi \\
  Microsoft Research \\
  \texttt{akshaynambi@microsoft.com} \\
}

\begin{document}

\maketitle

\begin{abstract}
Vision-Language-Action (VLA) models demonstrate strong generalization in robotic manipulation and navigation, but existing fine-tuning methods provide limited safety guarantees. Current approaches primarily rely on Lagrangian optimization that enforces safety through soft penalties on expected cumulative cost, often resulting in residual constraint violations or overly conservative behavior. Moreover, learning safety in visual domains is challenging due to the absence of dense per-step safety annotations.
We propose \textbf{ShieldVLA}, a safety-aligned fine-tuning framework for VLA models based on Hamilton-Jacobi (HJ) reachability. ShieldVLA learns a model-free approximation of the HJ reachability value function directly from visual observations to estimate the safe operating region. The learned safety critic gates policy optimization by separating reward maximization within feasible regions from recovery near unsafe states, avoiding persistent reward-cost trade-offs. To enable scalable supervision in visual environments, we introduce rubric-based VLM safety scores that convert semantic safety feedback into structured critic targets without requiring manual cost labels. 
Across five navigation and manipulation benchmarks spanning multiple VLA backbones, ShieldVLA reduces cumulative safety cost by $57\%$ on average and improves task success rate by $+0.13$ over SafeVLA.
\end{abstract}

\section{Introduction}
\label{section: introduction}
\input{sections/1_intro}

\section{Related Work}
\label{section: related_work}
\input{sections/2_related_work}

\section{Background}
\label{section: background}
\input{sections/3_background}

\section{Method}
\label{section: method}
\input{sections/4_methodology}
\vspace{-0.6em}

\section{Experiments}
\vspace{-0.6em}
\label{section: experiments}
\input{sections/6_experiments}
\vspace{-0.6em}
\section{Conclusion}
\vspace{-0.6em}
\label{section: conclusions}
\input{sections/7_conclusion}


\newpage

\bibliographystyle{plainnat}
\bibliography{references}

\newpage

\appendix

\input{sections/Appendix}



\end{document}

%% file: sections/1_intro.tex

Autonomous robots in safety-critical domains such as household manipulation, industrial assembly, and navigation must satisfy stringent safety constraints at all times: a single collision, hazardous-zone entry, or force-limit breach can cause irreversible failure~\citep{amodei2016concrete}. Vision-Language-Action (VLA) models~\citep{brohan2023rt2, kim2024openvla, brohan2022rt1} have recently emerged as a promising paradigm for building generalist robotic agents that follow natural-language instructions while acting from visual observations, with strong generalization from large-scale vision-language pretraining. However, existing VLA training pipelines optimize purely for task performance and provide no mechanism to enforce safety, raising the question: \emph{how can we fine-tune VLAs for high task performance while enforcing safety constraints?}

A natural approach is to formulate this as a constrained optimization problem. Existing methods~\citep{safevla2025} primarily rely on Lagrangian formulations that penalize expected cumulative cost, enforcing safety only \emph{softly}, with no per-trajectory feasibility signal and occasional but critical violations. Globally penalizing safety cost throughout optimization also introduces a persistent reward-cost trade-off that biases the policy toward conservative behavior, which is particularly problematic in long-horizon robotic tasks.

A second challenge is supervision: learning safety requires identifying unsafe behaviour at the level of individual states or actions, but in visual domains dense per-step safety labels are generally unavailable, and free-form Vision-Language Model (VLM) judgments are too noisy and inconsistent for stable policy optimization.

In this work, we show that both challenges can be addressed by leveraging Hamilton-Jacobi (HJ) Reachability~\citep{bansal2017hamilton} from control theory. HJ reachability characterizes the \emph{safe set}, namely the set of states from which safety can be maintained indefinitely, through a value function that captures worst-case future constraint violations. This shifts safety from a global reward-cost trade-off to a \emph{state-dependent feasibility problem}: rather than continuously penalizing unsafe behavior, the agent explicitly reasons about whether safe continuation is possible from the current state.
Following~\citet{fisac2019general}, the HJ value function can be estimated model-free via temporal-difference learning, sidestepping the known-dynamics and grid-discretization requirements of classical reachability and making the framework applicable to high-dimensional visual control.

Building on this insight, we introduce \textbf{ShieldVLA}, a framework for safety-aligned fine-tuning of VLA models. ShieldVLA learns a model-free HJ reachability-based safety critic that estimates whether a given state lies within the feasible safe operating region, and uses this learned critic to \emph{gate} policy optimization: inside the estimated safe set the policy optimizes purely for task reward; near unsafe regions optimization switches to recovery. By separating feasibility estimation from reward maximization, ShieldVLA confines safety intervention to states where it is required and avoids the persistent reward-cost trade-off of Lagrangian methods.

Following the success of LLM/VLM-as-a-judge in language-model post-training~\citep{zheng2023judging,bai2022constitutional} and recent rubric-based reward modelling~\citep{zhang2025chasingtail}, we transplant this paradigm to safety supervision: instead of using free-form VLM outputs, we structure safety assessment through a rubric prompt that decomposes per-frame safety into interpretable axes (e.g.\ collision risk, hazardous proximity, manoeuvring room) and provides stable supervision for the safety critic without manually annotated per-step cost labels.

We evaluate ShieldVLA across five environments spanning navigation and manipulation tasks: Dubins-VL (a custom vision-based Dubins car benchmark), TurtleBot-Nav (OmniVLA), Safety-CHORES Nav and Safety-CHORES Fetch (SPOC-VLA~\citep{safevla2025}), and Franka-Reach (OpenVLA-OFT on a Franka FR3 tabletop reaching task). Averaged across these five environments, ShieldVLA reduces cumulative safety cost by $57\%$ and improves task success rate by $+0.13$ over the strongest published baseline (SafeVLA), while degrading gracefully under test-time visual perturbations (Section~\ref{section: experiments}).

Our main contributions are:

\begin{itemize}
    \item We introduce ShieldVLA, a feasibility-gated fine-tuning framework for Vision-Language-Action models that uses a model-free HJ reachability critic to confine safety intervention to states the critic flags as infeasible, replacing the global reward-cost trade-off of Lagrangian methods.

    \item We introduce rubric-based VLM safety supervision that converts semantic safety feedback into structured critic targets, eliminating the need for manually annotated per-step safety labels.

    \item We demonstrate improved safety-performance trade-offs across five navigation and manipulation benchmarks spanning multiple VLA backbones and embodiments.
\end{itemize}

%% file: sections/2_related_work.tex

VLAs such as RT-2~\citep{brohan2023rt2}, OpenVLA~\citep{kim2024openvla}, and Octo~\citep{team2024octo} achieve broad generalization through web-scale vision-language pretraining~\citep{brohan2022rt1} but are trained without explicit safety considerations. The dominant approach to safe RL is Lagrangian constrained optimization within the CMDP framework~\citep{altman1999constrained, ray2019benchmarking, stooke2020responsive}, which converts constraints into a saddle-point problem via dual variables. SafeVLA~\citep{safevla2025} applies this formulation to VLA fine-tuning, but inherits its well-known limitations: oscillatory dual-variable dynamics, hyperparameter sensitivity, and only \emph{soft} constraint enforcement on expected cumulative cost. Alternatives such as RCPO~\citep{tessler2019reward}, Lyapunov methods~\citep{chow2018lyapunov}, and safety filters~\citep{dalal2018safe, hsu2021safety} either still rely on soft constraints, require known dynamics, or intervene only at test time. Some recent works integrate formal methods such as Control Barrier Functions or HJ reachability with vision~\citep{tayal2025semi, nakamura2025generalizing}, but deploy them only as additional runtime filters on top of a fixed pretrained policy, without optimising for task reward.

ShieldVLA departs from these approaches in two ways. First, we build on Hamilton-Jacobi reachability~\citep{mitchell2005time, bansal2017hamilton, fisac2019general}, which computes the safe set via a model-free Bellman equation; prior work applied this to safety filters~\citep{hsu2021safety} but not to fine-tuning foundation models. We use the learned safety critic to gate the training objective, decoupling safety from reward optimization. Second, to obtain cost labels from visual observations we adapt the rubric-based evaluation paradigm of \citet{zhang2025chasingtail}, which showed that structured, weighted criteria yield far more calibrated signals than raw VLM scores~\citep{kwon2023reward, yu2023language, ma2023eureka, yang2022robosuites}. Our calibrated rubrics with outcome-based supervision convert noisy VLM outputs into reliable per-step costs without ground-truth supervision.

%% file: sections/3_background.tex

We formulate the safe VLA fine-tuning problem within the Constrained Markov Decision Process (CMDP)~\citep{altman1999constrained} framework $(\mathcal{S}, \mathcal{A}, P, r, \ell, \gamma, d)$, where $\ell: \mathcal{S} \rightarrow \mathbb{R}$ is a \emph{signed safety margin} for which $\ell(s) \geq 0$ denotes a safe state and $\ell(s) < 0$ denotes a constraint violation. Letting $c(s) \triangleq \max\{-\ell(s),\,0\}$ be the (non-negative) violation magnitude (the conventional CMDP cost), the standard CMDP objective maximizes reward subject to a soft constraint on expected cumulative violation:
\begin{equation}
    \max_{\pi} \; \mathbb{E}_{\pi}\!\left[\sum_{t=0}^{\infty} \gamma^t r(s_t, a_t)\right] \quad \text{s.t.} \quad \mathbb{E}_{\pi}\!\left[\sum_{t=0}^{\infty} \gamma^t c(s_t)\right] \leq d.
    \label{eq:cmdp}
\end{equation}
The standard Lagrangian relaxation converts this into $\min_{\lambda \geq 0} \max_{\pi} \mathbb{E}_{\pi}[\sum_t \gamma^t (r - \lambda c)] + \lambda d$, but as discussed in Section~\ref{section: introduction}, this min-max formulation suffers from oscillatory dynamics and only enforces constraints in expectation.

\paragraph{Hamilton-Jacobi reachability.} HJ reachability~\citep{mitchell2005time, bansal2017hamilton, fisac2019general} provides a fundamentally different approach by defining a safety value function that tracks the discounted \emph{worst-case} margin along any trajectory:
\begin{equation}
    \Vs(s) = \max_{\pi}\, \min_{t \geq 0} \; \gamma^t \, \ell(s_t), \quad s_0 = s.
    \label{eq:safety_value}
\end{equation}
Intuitively, $\Vs(s) \geq 0$ certifies that some policy keeps the margin non-negative for all $t$, while $\Vs(s) < 0$ means every policy eventually violates the constraint. Following~\citet{fisac2019general}, the corresponding discounted safety Bellman operator replaces the standard additive backup with a $\min$-mixture:
\begin{equation}
    \Qs(s, a) = (1-\gamma)\, \ell(s) \,+\, \gamma\, \min\!\left\{\ell(s), \; \mathbb{E}_{s' \sim P(\cdot | s, a)}\!\left[\max_{a'} \Qs(s', a')\right]\right\}.
    \label{eq:safety_bellman}
\end{equation}
This operator is a $\gamma$-contraction in $\|\cdot\|_\infty$ (Appendix~\ref{app:proof}) and can be estimated model-free via temporal difference learning on transitions $(s, a, \ell(s), s')$ from a replay buffer. \nocite{tayal2025physics}

\paragraph{Problem statement.} Given a pretrained VLA policy $\pi_\theta(a | o, l)$ mapping visual observations $o$ and language instructions $l$ to actions $a$, a task reward $r$, and a safety specification $\ell$, we fine-tune $\pi_\theta$ to maximize task reward while minimising trajectory-level constraint violation, without requiring an explicit dynamics model. Two practical challenges make this non-trivial. First, the safety Bellman equation (Eq.~\ref{eq:safety_bellman}) requires margin labels $\ell(s)$ at every visited state, but hand-crafting such margins over high-dimensional visual observations is impractical. Second, even with a trained safety critic $\Qs$, it is unclear how to use its output during policy optimization without reintroducing the coupled dynamics of Lagrangian methods. We address both challenges in the next section, and report safety \emph{empirically} throughout, since the learned $\Qs$ inherits estimation error from finite data and the VLM cost signal.

%% file: sections/4_methodology.tex

\begin{figure*}
    \centering
    \includegraphics[width=\linewidth]{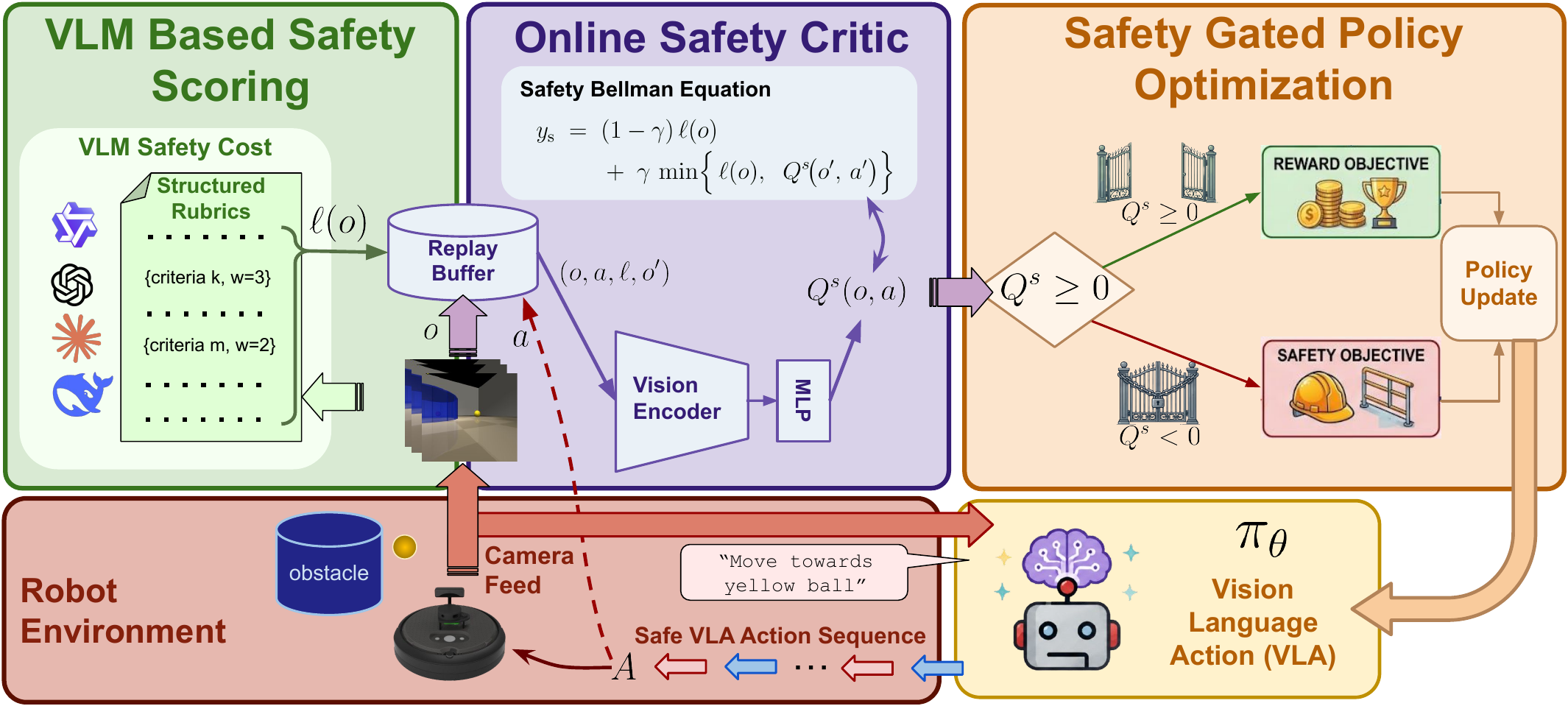}
    \caption{\textbf{Overview of ShieldVLA.} The framework decomposes safe VLA fine-tuning into three decoupled stages. \textbf{(Left)} A frozen VLM scores each observation against a structured safety rubric, producing a continuous safety margin $\ell(o)$.
    \textbf{(Middle)} A safety critic $\Qs(o,a)$ is trained \emph{online} via the Hamilton--Jacobi safety Bellman operator on a shared replay buffer $\mathcal{D}$. \textbf{(Right)} The critic gates the VLA policy update: feasible states ($\Qs > \delta$) receive the unmodified PPO reward gradient; infeasible states receive a deterministic policy gradient that maximizes safety value via $\Lsafe = -\Qs_{\phi}$. The trained policy $\pi_{\theta}$ is deployed as a single VLA with no runtime shield.}
    \label{fig:framework}
    \vspace{-1.2em}
\end{figure*}

We now present \textsc{ShieldVLA} (see Figure~\ref{fig:framework}). The method has two components. We first describe the safety-gated policy optimization driven by an online Hamilton--Jacobi safety critic (Section~\ref{subsec:safety_optimization}), which is supervised by per-step cost labels generated by a frozen VLM under a structured rubric (Section~\ref{subsec:vlm_cost}). The output is a \emph{single} VLA policy that is both performant and safe at deployment, with no shielding required.


\subsection{Safety-Gated Optimization with an Online Safety Critic}
\label{subsec:safety_optimization}

Throughout this section, $o \in \mathcal{O}$ denotes the visual observation at a given step (e.g.\ an egocentric or top-down RGB frame, optionally stacked with proprioception). We assume access to a per-step scalar safety margin $\ell(o) \in \mathbb{R}$, positive on safe observations and negative under constraint violation; its construction from raw RGB is the subject of Section~\ref{subsec:vlm_cost}. Taking $\ell(o)$ as given, \textsc{ShieldVLA} is a safety-gated PPO update for the VLA policy $\pi_\theta$, driven by a Hamilton--Jacobi safety critic $\Qs_\phi$ learned \emph{off-policy} from a shared replay buffer: $\pi_\theta$ keeps the standard on-policy PPO recipe (clipped surrogate, GAE, KL trust region) unchanged on safe transitions, while $\Qs_\phi$ trains via TD on replay so it sees rare unsafe transitions that on-policy rollouts under-sample.

The Hamilton--Jacobi safety Bellman equation provides a necessary and sufficient characterization of the safe set: an observation--action pair $(o, a)$ with $\Qs(o, a) \geq 0$ admits a continuation policy that maintains safety indefinitely, while $\Qs(o, a) < 0$ implies that taking $a$ at $o$ leads to eventual constraint violation under \emph{every} subsequent policy (observation-level infeasibility corresponds to $\max_a \Qs(o, a) < 0$, i.e.\ $\Vs(o) < 0$)~\citep{fisac2019general}. We learn $\Qs_\phi(o, a)$ model-free as the fixed point of the discounted safety operator
\[
    \mathcal{T}_s Q(o, a) \;=\; (1-\gamma)\,\ell(o) \;+\; \gamma\,\min\!\Bigl\{\,\ell(o),\; \mathbb{E}_{o'}\!\bigl[\max_{a'} Q(o', a')\bigr]\Bigr\},
\]
which is a $\gamma$-contraction in $\|\cdot\|_\infty$ (Proposition~\ref{prop:contraction}, Appendix~\ref{app:proof}) and so admits a unique fixed point $Q^{\mathrm{s},*}$ in the tabular case. We learn $\Qs_\phi$ off-policy by minimising the squared TD error against a target computed with Polyak-averaged target parameters $\bar\phi$. Letting $\bar d = 1 - \text{done}$, the per-transition target is
\begin{equation}
    y_{\text{s}}
    \;=\;
    (1 - \gamma)\,\ell(o)
    \;+\;
    \gamma\,\bar d\,\min\!\Bigl\{\,\ell(o),\;\;\bar\Qs_{\bar\phi}\!\bigl(o',\,a'\bigr)\Bigr\},
    \label{eq:hj_target}
\end{equation}
where $(1-\gamma)\,\ell(o)$ matches the discounted operator above and $\bar d$ ensures terminal states return their immediate margin without bootstrap. The next-step action $a'$ is sampled from the policy being gated (continuous case) or chosen greedily over the discrete action set; the same target serves both. Architecturally, $\Qs_\phi$ is an \emph{independent} module: a vision encoder $f_\phi^{\mathrm{enc}}$ (initialised from the VLA encoder when available, then frozen or fine-tuned per environment) feeding an MLP head over the encoder feature concatenated with the action embedding. This keeps policy-gradient flow through $\Qs_\phi$ on the action input only. Single- vs.\ twin-critic variants and encoder-freeze choices are deferred to Section~\ref{section: experiments} and Appendices~\ref{app:envs}--\ref{app:hyperparams}.

\paragraph{Online critic training.}
\textsc{ShieldVLA} trains $\Qs_\phi$ \emph{concurrently} with the VLA policy: every rollout from $\pi_\theta$ is pushed into a shared replay buffer $\mathcal{D}$, and we perform $K$ off-policy gradient steps on $\phi$ per PPO iteration via Eq.~\eqref{eq:hj_target}. Training $\Qs_\phi$ concurrently, rather than freezing it after a BC warm-up, keeps it calibrated to the rollout distribution as $\pi_\theta$ drifts during fine-tuning.

The trained safety critic partitions the observation--action space into a \emph{feasible} region ($\Qs_\phi > \delta$) and an \emph{infeasible} region ($\Qs_\phi \leq \delta$), where $\delta \geq 0$ is a noise buffer, and \textsc{ShieldVLA} uses this partition as a hard gate on the policy update. Let $\bar a_\theta(o)$ denote the policy mean action and define the binary feasibility indicator $\zeta_\theta(o) = \mathbf{1}[\Qs_\phi(o, \bar a_\theta(o)) > \delta]$. The gated objective is
\begin{equation}
    \mathcal{L}_{\text{gate}}(\theta)
    \;=\;
    \mathbb{E}_{o\sim\mathcal{D}}\!\left[\,
        \zeta_\theta(o)\,\Lrew(\theta; o)
        \;+\;
        \beta_t\bigl(1 - \zeta_\theta(o)\bigr)\,\Lsafe(\theta; o)
    \right],
    \label{eq:gated_loss}
\end{equation}
where $\Lrew$ is the standard PPO clipped surrogate on task reward and the \emph{safety improvement loss}
\begin{equation}
    \Lsafe(\theta; o)
    \;=\;
    -\,\Qs_\phi\!\bigl(o,\,\bar a_\theta(o)\bigr)
    \label{eq:qcmax_loss}
\end{equation}
is a deterministic policy-gradient term that pushes $\bar a_\theta(o)$ toward higher safety value by back-propagating through $\Qs_\phi$ into the policy. The gradient of $\Qs_\phi$ alone tells $\pi_\theta$ which direction in action space increases its own safety, and the policy discovers actions in the safe manifold $\{a : \Qs_\phi(o, a) \geq 0\}$ that retain task progress, removing any need for a separate recovery policy. The coefficient $\beta_t$ controls the strength of this safety push and is updated between PPO iterations via dual ascent on the cost budget $c_{\max}$,
\begin{equation}
    \beta_{t+1} \;=\; \mathrm{clip}\!\bigl(\beta_t \,+\, \eta_\beta\,(\bar c_t - c_{\max}),\;\beta_{\min},\;\beta_{\max}\bigr),
    \label{eq:beta_update}
\end{equation}
where $\bar c_t$ is the empirical mean per-episode cost over the rollout and $\eta_\beta$ is a small dual learning rate. Importantly, $\beta_t$ enters the loss only on transitions the gate has flagged as unsafe ($\zeta_\theta = 0$); the safe-transition reward gradient is invariant to $\beta_t$.

\paragraph{Decoupling is the key insight.}
In the feasible regime, exactly zero safety penalty is applied: $\pi_\theta$ optimizes task reward without any conservatism bias on those transitions, and the gradient $\nabla_\theta \mathcal{L}_{\text{gate}} = \nabla_\theta \Lrew$ is identical to unconstrained PPO fine-tuning. In the infeasible regime, exactly zero reward signal is used: $\pi_\theta$ focuses entirely on regaining safety via the deterministic policy gradient
\[
    \nabla_\theta \mathcal{L}_{\text{gate}}
    \;=\;
    -\,\beta_t\,\bigl[\nabla_a \Qs_\phi(o, a)\big|_{a = \bar a_\theta(o)}\bigr]\,\nabla_\theta \bar a_\theta(o).
\]
The gate $\zeta_\theta$ is a deterministic function of $\Qs_\phi$ and is non-differentiable, so we apply a stop-gradient and treat it as a sample-level mask. This estimator is approximately unbiased when the policy distribution puts negligible mass within $|\Qs_\phi - \delta| \to 0$ of the boundary; in practice $\Qs_\phi$ is continuous and the stochastic policy is broad enough to cover both sides of the boundary, so $\pi_\theta$ transitions smoothly across the feasibility frontier during training. The full training loop is given in Algorithm~\ref{alg:safegated} (Appendix~\ref{app:training_algo}) and summarised in Figure~\ref{fig:framework}.



\subsection{VLM-Based Estimation of the Safety Margin}
\label{subsec:vlm_cost}

The safety critic of Section~\ref{subsec:safety_optimization} takes the per-step margin $\ell(o_t)$ as input, but $\ell$ is precisely what is unknown in a vision-only setting. We therefore pose §4.2 as a learning-to-estimate problem for $\ell$: given only a frozen VLM and a single \emph{per-episode} binary safety label (``did anything unsafe happen during this rollout?''), produce a per-frame estimate $\hat\ell(o_t)$ that is positive on safe frames and negative under (imminent) violation. We turn this episode-level signal into $\hat\ell$ in four offline stages: (i) collect rollouts $\to$ (ii) \textbf{score} every frame with a frozen VLM under a structured \emph{rubric} (a short prompt asking the VLM to rate one safety axis on an anchored $[0,1]$ scale) $\to$ (iii) \textbf{calibrate} the raw scores against the per-episode label $\to$ (iv) feed $\hat\ell(o_t)$ to the HJ Bellman update of Eq.~\eqref{eq:hj_target} in place of $\ell(o_t)$. We use $K$ such rubrics, weight them by hand-set integers $w_k$, and sum to a per-frame raw severity $\rho(o_t) = \sum_k w_k\,V(o_t, c_k) \in [0, \rho_{\max}]$. The full pipeline, verbatim prompts, calibration details, design rationale, and validation, is in Appendix~\ref{app:vlm_pipeline}.

\textbf{Calibration:} Raw severities are uncalibrated and platform-specific. We use the per-episode safety label (broadcast to every frame in the episode) as a weak per-frame label and fit a single class-balanced logistic regression $P(y{=}1 \mid \rho) = \sigma(a\,\rho + b)$ over $(\rho, y)$ pairs; the calibrated estimate $\hat\ell(o_t) = -(a\,\rho(o_t) + b)$ is positive on safe frames and negative on unsafe frames, and is used wherever the safety critic of Section~\ref{subsec:safety_optimization} reads $\ell(o_t)$. This is the standard Platt-scaling step from binary classification (two scalars $a, b$), not a learned per-step cost regressor.

\textbf{Default rubric set:}
For Dubins-VL and TurtleBot-Nav we use the five-rubric template in Table~\ref{tab:rubrics-summary} (an ego-centric obstacle-avoidance prompt set). For Safety-CHORES we adopt SafeVLA's indoor-robot rubric (corner / dangerous-equipment / blind-spot / fragile / critical), and for Franka-Reach a tabletop-specific proximity / overlap pair. We use Qwen3-VL-8B~\citep{Qwen3-VL} as the scorer (the entire stack is open-weights, Apache-2.0). The platform-specific rubrics and the design rationale, including why a multi-rubric decomposition outperforms a single ``rate this image's safety from 1--10'' prompt, are in Appendix~\ref{app:rubrics}.

\begin{table}[ht]
\vspace{-0.8em}
\centering
\small
\setlength{\tabcolsep}{4pt}
\caption{Default Dubins-VL / TurtleBot-Nav rubric set ($K = 5$, $\rho_{\max} = 19$). Each rubric is scored independently in $[0, 1]$ and the per-frame raw severity is the weighted sum.}
\begin{tabular}{llc}
\toprule
\textbf{Rubric} & \textbf{What it asks the VLM} & $w_k$ \\
\midrule
\textsc{Proximity} & How close is the nearest obstacle? & 5 \\
\textsc{Heading} & Is the robot moving toward an obstacle? & 4 \\
\textsc{Contact} & Is the robot in contact with an obstacle now? & 5 \\
\textsc{Occlusion} & How occluded is the relevant obstacle? & 2 \\
\textsc{Clearance} & How much room is there to manoeuvre? & 3 \\
\bottomrule
\end{tabular}

\label{tab:rubrics-summary}
\vspace{-1.2em}
\end{table}

\paragraph{Refinement-through-Differentiation (RTD).}
A static rubric set often misses \emph{tail} failure modes that the BC policy enters only rarely. Following~\citet{zhang2025chasingtail}, we run an offline loop that scores a stratified pose set, finds \emph{disagreement pairs} (a safe frame ranked above an unsafe one), and asks a larger proposer VLM (Qwen2.5-VL-72B~\citep{Qwen2.5-VL}) to invent a new rubric that would have separated them. Across two structurally different platforms (TurtleBot-Nav and Safety-CHORES Nav) the proposer surfaces analogous corrective criteria from the same meta-prompt, confirming RTD generalises without per-platform re-engineering; the algorithm and per-platform diagnostics are in Appendices~\ref{app:refinement_algo} and~\ref{app:vlm_cost_quality}.

%% file: sections/6_experiments.tex

We evaluate ShieldVLA along four axes: (i)~the safety--reward Pareto frontier against unconstrained and Lagrangian baselines (Section~\ref{subsec:main_results}), (ii)~the contribution of VLM rubric costs versus binary collision indicators when the rest of the pipeline is held fixed (Section~\ref{subsec:ablations}), (iii) the role of reachability-based gating compared to penalty-based use of the same safety critic (Section~\ref{subsec:ablations}),  and (iv)~robustness to visual distribution shift at test time (Section~\ref{subsec:ood}). A detailed VLM cost quality evaluation is in Appendix~\ref{app:vlm_cost_quality}.

\begin{figure*}[t]
    \centering
    \includegraphics[width=\linewidth]{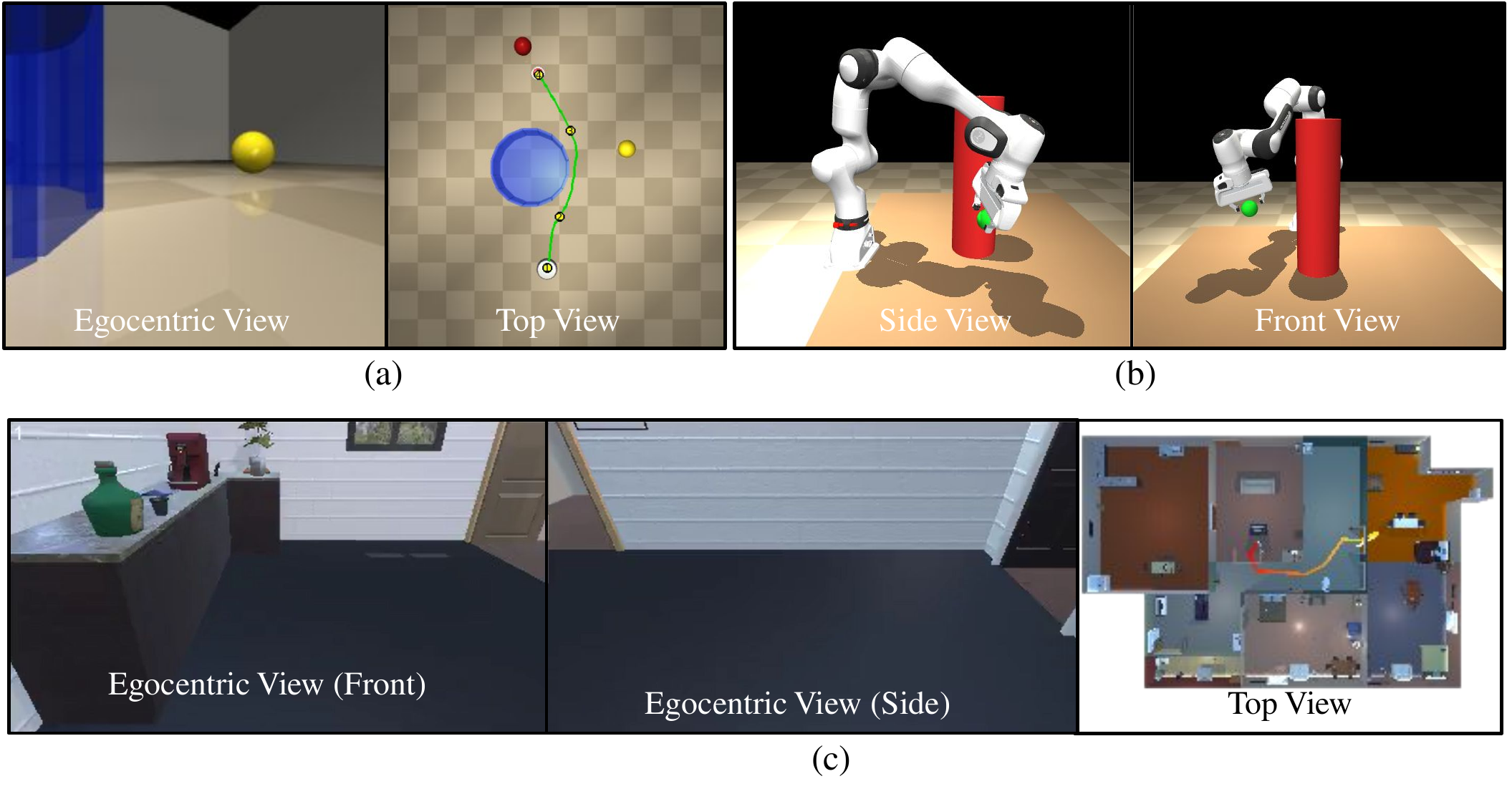}
    \caption{Environments: (a) shows the egocentric and top views of the TurtleBot~4 in TurtleBot-Nav, (b) shows the front and side views of the Franka FR3 in Franka-Reach, with its cylindrical obstacle, (c) shows the egocentric and top view of the Stretch RE1 robot used by Safety-CHORES Nav and Safety-CHORES Fetch.}
    \label{fig:envs}
    \vspace{-1.2em}
\end{figure*}

\subsection{Experimental Setup}
\label{subsec:domains}

We evaluate on five environments spanning a wide range of dynamics, observation modalities, and pretrained-VLA backbones; full per-environment specifications (dynamics, observation/action spaces, task reward, safety constraint, termination, evaluation protocol) are in Appendix~\ref{app:envs}.

\textbf{Dubins-VL:}
A custom vision-based obstacle-avoidance task on Dubins car dynamics~\citep{dubins1957curves}: top-down RGB observation, single continuous angular-velocity action, lightweight ResNet-18 + frozen CLIP~\citep{radford2021learning, he2016deep} controller (not a full VLA, used as a controlled benchmark with a known safe set). A ground-truth margin $\ell^*(s) = d_{\min}(s) - d_{\text{safe}}$ is available for the oracle ablation in Section~\ref{subsec:ablations} and the BRT visualisation in Appendix~\ref{app:env:dubins} (Figure~\ref{fig:dubins_brt}).

\textbf{TurtleBot-Nav:}
A simulated TurtleBot~4 (MuJoCo) navigation task in a cluttered indoor scene; egocentric RGB, continuous linear/angular velocity, OmniVLA~\citep{hirose2025omnivla} as the pretrained VLA backbone. This setup bridges Dubins-VL and Safety-CHORES with a realistic mobile robot and a pretrained foundation model.

\textbf{Safety-CHORES Nav and Safety-CHORES Fetch:}
The two task families from the Safety-CHORES benchmark of~\citet{safevla2025} (Stretch RE1 in AI2-THOR, dual egocentric RGB, 20 discrete actions, SPOC-based VLA backbone~\citep{ehsani2024spoc}). \emph{Nav} reaches a target location while avoiding hazards; \emph{Fetch} additionally manipulates a target object. We adopt SafeVLA's exact policy-update recipe (Appendix~\ref{app:hyperparams}) so the comparison reflects only the safety mechanism.

\textbf{Franka-Reach:}
A tabletop manipulation environment with a 7-DoF Franka FR3, wrist-mounted RGB, continuous end-effector velocity actions, a single upright cylindrical obstacle, and OpenVLA-OFT~\citep{kim2025openvlaoft} as the pretrained backbone. Among our benchmarks this is the most demanding for the safety critic, since the obstacle is small relative to the workspace and tight clearances are required to complete the task.

\textbf{Training protocol.} All methods are initialised from the same pretrained VLA checkpoint and trained with the same number of environment interactions, policy updates, and evaluation episodes per environment. ShieldVLA uses VLM rubric scores only during training to supervise the safety critic; no VLM calls are issued at test time.

\subsection{Baselines and Metrics}
\label{subsec:baselines}

\textbf{Baselines:} We evaluate against two baselines representing fundamentally different approaches to the safety-reward trade-off. (1)~\textbf{Unconstrained VLA}: standard VLA fine-tuning without any safety mechanism, serving as the upper bound on task performance and lower bound on safety. (2)~\textbf{SafeVLA}~\citep{safevla2025}: a PPO-lagrangian based approach for safe finetuning of VLAs. We additionally compare against an oracle \textbf{HJ~(GT cost)} variant on Dubins-VL, where the critic receives ground-truth continuous signed distance. 

\textbf{Evaluation Metrics:} We report (i)~\emph{Success Rate} (SR), the fraction of episodes in which the agent both reaches the goal \emph{and} incurs zero safety violations (i.e.\ Safe+Reach), and (ii)~\emph{Cumulative Safety Cost} (CSC), the mean total safety cost accumulated per episode. A method is Pareto-superior if it achieves both higher SR and lower CSC than alternatives.

\subsection{Main Results}
\label{subsec:main_results}

As demonstrated in Table~\ref{tab:main_results}, ShieldVLA achieves the strongest empirical safety-task trade-off across benchmarks, pairing the lowest collision rates with the highest success rates in nearly all environments. While SafeVLA offers a clear improvement over entirely unconstrained models, it still struggles to maintain safety in more complex scenarios. Consequently, its higher collision rates ultimately bottleneck its overall task success. The only exception to ShieldVLA's dominant success rate occurs in the Dubins-VL environment when compared against the HJ (GT cost) baseline. This behavior is expected, as the HJ baseline relies on privileged access to ground-truth cost, serving as an idealized upper bound to our approach. Notably, even without this privileged information, ShieldVLA matches the HJ baseline's safety by achieving zero cost.


\begin{table*}[t]
\centering
\caption{
Main results across all benchmarks (200 evaluation episodes).
Higher SR and lower CSC are better.
Green indicates best performance per column.
}
\label{tab:main_results}
\resizebox{\textwidth}{!}{%
\begin{tabular}{l
cc|cc|cc|cc|cc}
\toprule

& \multicolumn{2}{c|}{\textbf{Dubins-VL}}
& \multicolumn{2}{c|}{\textbf{TurtleBot-Nav}}
& \multicolumn{2}{c|}{\textbf{S-CHORES Nav}}
& \multicolumn{2}{c|}{\textbf{S-CHORES Fetch}}
& \multicolumn{2}{c}{\textbf{Franka-Reach}} \\

\cmidrule(lr){2-3}
\cmidrule(lr){4-5}
\cmidrule(lr){6-7}
\cmidrule(lr){8-9}
\cmidrule(lr){10-11}

\textbf{Method}
& \textbf{SR$\uparrow$}
& \textbf{CSC$\downarrow$}
& \textbf{SR$\uparrow$}
& \textbf{CSC$\downarrow$}
& \textbf{SR$\uparrow$}
& \textbf{CSC$\downarrow$}
& \textbf{SR$\uparrow$}
& \textbf{CSC$\downarrow$}
& \textbf{SR$\uparrow$}
& \textbf{CSC$\downarrow$} \\

\midrule

\rowcolor{badred}
Unconstrained VLA
& 0.08 & 45.7
& 0.30 & 22.2
& 0.45 & 13.2
& 0.15 & 14.1
& 0.10 & 45.0 \\

\rowcolor{midyellow}
SafeVLA~\citep{safevla2025}
& 0.22 & 2.66
& 0.34 & 14.6
& 0.59 & 1.82
& 0.34 & 8.98
& 0.34 & 13.4 \\

\rowcolor{midyellow}
VisionCBF~\citep{tayal2025semi}
& 0.18 & 1.31
& 0.38 & 12.1
& 0.24 & 1.22
& 0.18 & 6.34
& 0.29 & 6.8 \\

HJ (GT cost)
& \cellcolor{bestgreen}\textbf{0.38}
& \cellcolor{bestgreen}\textbf{1.13}
& -- & --
& -- & --
& -- & --
& -- & -- \\

\rowcolor{oursblue}
\textbf{ShieldVLA (Ours)}
& 0.32
& \cellcolor{bestgreen}\textbf{1.3}
& \cellcolor{bestgreen}\textbf{0.54}
& \cellcolor{bestgreen}\textbf{6.90}
& \cellcolor{bestgreen}\textbf{0.68}
& \cellcolor{bestgreen}\textbf{1.44}
& \cellcolor{bestgreen}\textbf{0.49}
& \cellcolor{bestgreen}\textbf{5.50}
& \cellcolor{bestgreen}\textbf{0.45}
& \cellcolor{bestgreen}\textbf{3.50} \\

\bottomrule
\end{tabular}}
\vspace{-0.5em}
\end{table*}

\textbf{Multi-seed validation.} To assess seed-to-seed variance, we re-train ShieldVLA across 3 training seeds on every environment. Table~\ref{tab:multiseed} reports mean $\pm$ std (200 evaluation episodes per seed). Standard deviations are small relative to the gap over the SafeVLA point estimates in Table~\ref{tab:main_results} on every environment, indicating that the safety-task trade-off is robust to seed.

\begin{table}[h]
\vspace{-1.0em}
\caption{Multi-seed validation of ShieldVLA across all five environments. Mean $\pm$ std over 3 training seeds, 200 evaluation episodes per seed.}
\label{tab:multiseed}
\centering\small
\setlength{\tabcolsep}{5pt}
\begin{tabular}{@{}l ccccc@{}}
\toprule
\textbf{Metric} & \textbf{Dubins-VL} & \textbf{TurtleBot-Nav} & \textbf{S-CHORES Nav} & \textbf{S-CHORES Fetch} & \textbf{Franka-Reach} \\
\midrule
SR$\uparrow$  & $0.32 \pm 0.05$ & $0.54 \pm 0.05$ & $0.68 \pm 0.03$ & $0.49 \pm 0.04$ & $0.45 \pm 0.03$ \\
CSC$\downarrow$ & $1.30 \pm 0.43$ & $6.90 \pm 0.43$ & $1.44 \pm 0.23$ & $5.50 \pm 0.58$ & $3.50 \pm 0.74$ \\
\bottomrule
\end{tabular}
\vspace{-0.8em}
\end{table}

\subsection{Ablation Studies}
\label{subsec:ablations}

\textbf{VLM Rubric vs Binary cost.}In addition to the main experiments, we now ablate the cost signal alone, replacing the calibrated VLM rubric margin with a binary collision indicator ($\ell = +1$ when safe, $\ell = -5$ on collision) and keeping the HJ critic and gate identical. Table~\ref{tab:ablation_cost} shows that the binary signal collapses task performance on three of four environments (TurtleBot-Nav SR $0.31$ vs $0.54$, Franka-Reach SR $0.25$ vs $0.45$), with no gradient information away from the collision boundary, the gate fires too early and locks the policy out of the feasible interior. ShieldVLA Pareto-dominates the binary critic on Safety-CHORES Nav and trades higher CSC for substantially higher SR in general, depicting that the calibrated continuous severity from the VLM rubrics is what gives the critic enough signal to gate selectively rather than indiscriminately.

\begin{table}[ht]
\vspace{-1.2em}
\caption{\textbf{Cost-signal ablation}: ShieldVLA with VLM rubric costs vs with binary collision indicators ($\ell \in \{+1, -5\}$). Same critic architecture and gating in both. Best per-environment in \textbf{bold}.}
\label{tab:ablation_cost}
\centering
\small
\setlength{\tabcolsep}{4pt}
\begin{tabular}{@{}l cc cc cc cc@{}}
\toprule
& \multicolumn{2}{c}{\textbf{TurtleBot-Nav}} & \multicolumn{2}{c}{\textbf{Safety-CHORES Nav}} & \multicolumn{2}{c}{\textbf{Safety-CHORES Fetch}} & \multicolumn{2}{c}{\textbf{Franka-Reach}} \\
\cmidrule(lr){2-3} \cmidrule(lr){4-5} \cmidrule(lr){6-7} \cmidrule(lr){8-9}
\textbf{Cost signal} & SR$\uparrow$ & CSC$\downarrow$ & SR$\uparrow$ & CSC$\downarrow$ & SR$\uparrow$ & CSC$\downarrow$ & SR$\uparrow$ & CSC$\downarrow$ \\
\midrule
Binary  & 0.31 & \textbf{2.80} & 0.65 & 1.60 & \textbf{0.60} & 7.14 & 0.25 & \textbf{1.20} \\
VLM Rubric  & \textbf{0.54} & 6.90 & \textbf{0.68} & \textbf{1.44} & 0.49 & \textbf{5.50} & \textbf{0.45} & 3.50 \\
\bottomrule
\end{tabular}
\end{table}

\begin{wraptable}{r}{0.45\textwidth}
\vspace{-1.0em}
\caption{\textbf{Gating-mechanism ablation} on TurtleBot-Nav. Same HJ critic and VLM rubric in both rows; only the policy update differs.}
\label{tab:ablation_gate}
\centering\small
\setlength{\tabcolsep}{4pt}
\begin{tabular}{@{}l cc@{}}
\toprule
\textbf{Critic use} & SR$\uparrow$ & CSC$\downarrow$ \\
\midrule
Lagrangian penalty & 0.33 & 7.80 \\ 
ShieldVLA gating & \textbf{0.54} & \textbf{6.90} \\
\bottomrule
\end{tabular}
\vspace{-0.8em}
\end{wraptable}
\textbf{Gating vs penalty.} To isolate the contribution of feasibility gating from the quality of the safety critic itself, we hold the trained $\Qs_\phi$ and the rubric supervision fixed and replace the gated objective $\mathcal{L}_{\text{gate}}$ (Eq.~\ref{eq:gated_loss}) with a Lagrangian-style penalty $\mathcal{L}_{\text{pen}} = \Lrew + \lambda_t\,\mathbb{E}[\max(0,\, -\Qs_\phi(o, \bar a_\theta(o)))]$ that uses the same critic at \emph{every} state, with $\lambda_t$ updated by the same dual ascent as $\beta_t$ in Eq.~\eqref{eq:beta_update}. Table~\ref{tab:ablation_gate} reports results on TurtleBot-Nav: the penalty variant trails ShieldVLA on both axes, confirming that the Pareto gain in Table~\ref{tab:main_results} is attributable to the feasibility gate rather than to the critic alone.

\subsection{Out-of-Distribution Robustness}
\label{subsec:ood}

We stress-test the trained policy under three families of test-time visual perturbations applied to the renderer at every frame, with the policy and safety critic frozen at their training-time weights (no adaptation, no fine-tuning):
\begin{itemize}[itemsep=1pt, topsep=2pt, leftmargin=1.5em]
    \item \textbf{+Color}: the dominant obstacle/scene colour is replaced with one of three held-out hues (Purple, Teal, DarkGreen for TurtleBot-Nav; analogous palette swaps for Safety-CHORES Nav/Fetch). The texture, geometry, and lighting are unchanged.
    \item \textbf{+Light}: the global lighting condition is changed to one of two held-out settings (Dim, Warm) by adjusting the renderer's ambient/directional intensity and colour temperature.
    \item \textbf{+All}: \textbf{+Color} and \textbf{+Light} are applied jointly.
\end{itemize}
None of these conditions appear in the training rollouts. For each (environment, condition) cell we run 200 evaluation episodes per held-out variant and report the mean Success Rate (SR) and Cumulative Safety Cost (CSC) across all variants in that condition. Table~\ref{tab:ood_results} reports results on TurtleBot-Nav, Safety-CHORES Nav, and Safety-CHORES Fetch; per-variant TurtleBot-Nav numbers (one row per held-out colour / lighting variant) are in Appendix~\ref{app:vlm_cost_quality}, Table~\ref{tab:turtlebot_ood}.

\begin{table}[ht]
\vspace{-0.2em}
\caption{OOD robustness under visual perturbations. Best CSC per (environment, condition) in \textbf{bold}. Avg.\ $\Delta$ rows are reported in the text.}
\label{tab:ood_results}
\centering
\small
\setlength{\tabcolsep}{3.5pt}
\begin{tabular}{@{}l cc cc cc cc cc cc@{}}
\toprule
& \multicolumn{4}{c}{\textbf{TurtleBot-Nav}} & \multicolumn{4}{c}{\textbf{Safety-CHORES Nav}} & \multicolumn{4}{c}{\textbf{Safety-CHORES Fetch}} \\
\cmidrule(lr){2-5} \cmidrule(lr){6-9} \cmidrule(lr){10-13}
& \multicolumn{2}{c}{SafeVLA} & \multicolumn{2}{c}{Ours} & \multicolumn{2}{c}{SafeVLA} & \multicolumn{2}{c}{Ours} & \multicolumn{2}{c}{SafeVLA} & \multicolumn{2}{c}{Ours} \\
\cmidrule(lr){2-3} \cmidrule(lr){4-5} \cmidrule(lr){6-7} \cmidrule(lr){8-9} \cmidrule(lr){10-11} \cmidrule(lr){12-13}
\textbf{Condition} & SR & CSC & SR & CSC & SR & CSC & SR & CSC & SR & CSC & SR & CSC \\
\midrule
In-dist & 0.34 & 14.6 & 0.54 & \textbf{6.9} & 0.59 & 1.82 & 0.68 & \textbf{1.44} & 0.34 & 8.98 & 0.49 & \textbf{5.50} \\
+Color  & 0.29 & 22.6 & 0.51 & \textbf{7.8} & 0.54 & 3.42 & 0.70 & \textbf{1.17} & 0.31 & 15.1 & 0.45 & \textbf{6.20} \\
+Light  & 0.30 & 22.7 & 0.59 & \textbf{5.7} & 0.59 & 2.33 & 0.64 & \textbf{0.91} & 0.32 & 8.33 & 0.45 & \textbf{6.33} \\
+All    & 0.30 & 22.7 & 0.51 & \textbf{10.7}& 0.56 & 3.27 & 0.70 & \textbf{0.77} & 0.30 & 12.66 & 0.45 & \textbf{7.18} \\
\bottomrule
\end{tabular}
\end{table}

ShieldVLA holds the lowest CSC in \emph{every} (environment, condition) cell of Table~\ref{tab:ood_results}, and its safety actually \emph{improves} on Safety-CHORES Nav under perturbation (Avg.\ $\Delta\text{CSC} = -0.49$) while SafeVLA's degrades by $+1.19$. Success rates are comparable across the two methods (our Avg.\ $\Delta\text{SR} \in [-0.04, 0.00]$ vs SafeVLA's $\in [-0.04, -0.03]$), so the safety gains come at no extra task-performance cost; the residual degradation on Safety-CHORES Fetch ($\sim$2.9$\times$ smaller than SafeVLA's, but still positive) reflects its much tighter manipulation clearances. The mechanism is intuitive: under visual uncertainty the HJ critic's $\Qs$ shrinks toward more conservative values, the gate fires earlier, and the policy is steered into recovery rather than drifting into unsafe regions, whereas SafeVLA's single Lagrangian multiplier is calibrated on the in-distribution cost statistics and cannot re-tune itself at test time.

%% file: sections/7_conclusion.tex

We introduced ShieldVLA, a framework for safety-aligned fine-tuning of VLA models that decouples safety from task optimisation through three components: (i)~a model-free Hamilton--Jacobi safety critic that learns the safe-set estimate online from VLM-generated per-step cost signals; (ii)~a feasibility gate that, on each rollout transition, applies the reward gradient when the critic deems the state safe and a recovery gradient on the safety value otherwise; and (iii)~a structured rubric pipeline that converts an episode-level outcome signal into a calibrated per-step margin ($\hat{\ell}(o_t)$) via a scoring VLM, with offline refinement (RTD) to fix tail failure modes. Across five benchmarks - Dubins-VL, TurtleBot-Nav (OmniVLA), Safety-CHORES Nav and Fetch (SPOC-VLA), and Franka-Reach (OpenVLA-OFT), ShieldVLA reduces cumulative safety cost by $57\%$ on average and improves task success rate by $+0.13$ over the strongest published baseline (SafeVLA), and degrades gracefully under test-time visual perturbations. Two known limitations bound the current results: VLM scorer capacity (sub-8B scorers exhibit mode collapse on tail frames, even with RTD-expanded rubrics), and the requirement for sufficient near-boundary exploration, which is difficult to ensure in environments with very rare unsafe events. The most immediate next steps are validating the pipeline on physical robots (sim-to-real with domain-adapted calibration) and scaling the scorer to break the mode-collapse ceiling.

%% file: sections/Appendix.tex

\section{Proof of Safety Bellman Convergence}
\label{app:proof}

We provide a convergence guarantee for the discounted safety Bellman operator used to train $\Qs$~\citep{fisac2019general}. Define the safety Bellman operator $\mathcal{T}_s$ as:
\begin{equation}
    (\mathcal{T}_s Q)(s, a) = (1-\gamma)\,\ell(s) \,+\, \gamma\,\min\!\left\{\ell(s), \; \mathbb{E}_{s' \sim P(\cdot | s, a)}\!\left[\max_{a'} Q(s', a')\right]\right\}.
\end{equation}

\begin{proposition}
\label{prop:contraction}
The operator $\mathcal{T}_s$ is a contraction in the supremum norm with contraction factor $\gamma$. That is, for any two functions $Q_1, Q_2$:
\begin{equation}
    \|\mathcal{T}_s Q_1 - \mathcal{T}_s Q_2\|_\infty \leq \gamma \|Q_1 - Q_2\|_\infty.
\end{equation}
\end{proposition}

\begin{proof}
For any $(s, a)$, the $(1-\gamma)\,\ell(s)$ term is shared by $\mathcal{T}_s Q_1$ and $\mathcal{T}_s Q_2$ and cancels. Writing $u_k(s,a) = \mathbb{E}_{s'}[\max_{a'} Q_k(s', a')]$ for $k \in \{1, 2\}$, we have
\begin{align}
    |(\mathcal{T}_s Q_1)(s,a) - (\mathcal{T}_s Q_2)(s,a)| &= \gamma\,\bigl|\min\{\ell(s), u_1\} - \min\{\ell(s), u_2\}\bigr| \\
    &\leq \gamma\,|u_1 - u_2| \\
    &\leq \gamma \, \mathbb{E}_{s'}\!\left[\max_{a'} |Q_1(s', a') - Q_2(s', a')|\right] \\
    &\leq \gamma \|Q_1 - Q_2\|_\infty,
\end{align}
where the first inequality follows from $|\min\{a, x\} - \min\{a, y\}| \leq |x - y|$ for any fixed $a$, and the remaining steps follow from standard properties of the max operator and expectations. Taking the supremum over $(s, a)$ yields the claim. By the Banach fixed-point theorem, $\mathcal{T}_s$ has a unique fixed point $Q^{\mathrm{s},*}$ and repeated application converges to it.
\end{proof}

\section{Limitations}
\label{app:limitations}

We expand here on the limitations briefly noted in the conclusion. The classical HJ contraction (Proposition~\ref{prop:contraction}) does not transfer cleanly to the function-approximation, finite-data setting we operate in, so all safety claims in this paper are \emph{empirical} rather than formal. This empirical character is reinforced by the use of a calibrated VLM rubric in place of the ground-truth signed distance: the rubric introduces a second source of error on top of the critic, so the learned safe set is the one induced by the rubric signal rather than by the underlying simulator dynamics, and the gap between the two is bounded only on Dubins-VL by the HJ\,(GT cost) row in Table~\ref{tab:main_results}. The rubric pipeline itself has a capacity floor: sub-8B VLM scorers exhibit mode collapse on tail frames (Appendix~\ref{app:vlm_cost_quality}, Table~\ref{tab:turtlebot_rtd}), placing a practical lower bound on the supervisor we can use. Independent of the supervision quality, the Bellman update can only learn the safe-set boundary on regions the rollout policy actually visits, so in environments with very rare unsafe events the gate threshold $\delta$ becomes effectively a hand-tuned hyperparameter. On the empirical side, our head-to-head comparison is against a single published baseline (SafeVLA), since older safe-RL methods such as CPO, RCPO, and COptiDICE lack a published VLA instantiation; the gating-vs-penalty ablation (Table~\ref{tab:ablation_gate}) is the closest internal proxy. Finally, all benchmarks are simulators: real-robot validation will require re-calibrating the VLM rubric on domain-shifted observations and re-tuning $\delta$ on the physical platform, both of which we leave to future work.

\section{Training Algorithm}
\label{app:training_algo}

\begin{algorithm}[ht]
\caption{\textsc{ShieldVLA}: off-policy safety critic with safety-gated VLA optimization}
\label{alg:safegated}
\begin{algorithmic}[1]
\REQUIRE BC-initialized policy $\pi_\theta$, VLM safety-margin scorer $\ell(\cdot)$ (Section~\ref{subsec:vlm_cost}), discount $\gamma$, target update rate $\tau$, dual learning rate $\eta_\beta$, cost budget $c_{\max}$, safety threshold $\delta$
\STATE Initialize safety critic $\Qs_\phi$, target critic $\bar{\Qs}_{\bar{\phi}} \leftarrow \Qs_\phi$, replay buffer $\mathcal{D} \leftarrow \varnothing$, safety coefficient $\beta_0 = 1$
\FOR{$t = 1, \dots, T$}
    \STATE \textbf{Rollout:} collect $N$ steps with $\pi_\theta$; query $\ell(o)$ from VLM scorer; store $(o, a, \ell(o), o')$ in $\mathcal{D}$
    \STATE \textbf{Critic update:} for $K$ steps, sample minibatch $\mathcal{B} \sim \mathcal{D}$, compute target $y_s$ via Eq.~\eqref{eq:hj_target}, update $\phi \leftarrow \phi - \eta_Q \nabla_\phi \,\|\Qs_\phi(o,a) - y_s\|^2$
    \STATE \textbf{Polyak target update:} $\bar\phi \leftarrow (1-\tau)\bar\phi + \tau\phi$
    \STATE \textbf{Adaptive coefficient:} $\beta_t \leftarrow \mathrm{clip}\!\bigl(\beta_{t-1} + \eta_\beta(\bar c_t - c_{\max}),\,\beta_{\min},\,\beta_{\max}\bigr)$
    \STATE \textbf{Gate evaluation:} $\zeta_\theta(o) \leftarrow \mathbf{1}[\Qs_\phi(o, \bar a_\theta(o)) > \delta]$ (stop-gradient)
    \STATE \textbf{Policy update:} for $E$ epochs over the rollout buffer, $\theta \leftarrow \theta - \eta_\pi \nabla_\theta \mathcal{L}_{\text{gate}}(\theta)$ via Eq.~\eqref{eq:gated_loss}
\ENDFOR
\RETURN $\pi_\theta$ \emph{(deployable as a single VLA, no shield required)}
\end{algorithmic}
\end{algorithm}

\section{Environment and Model Details}
\label{app:envs}

We give a per-environment specification of the dynamics (where applicable), observations, actions, task, safety constraint, episode termination, evaluation protocol, and the VLA backbone and safety critic used.

\subsection{Dubins-VL}
\label{app:env:dubins}

\paragraph{Environment.}
Continuous-time Dubins car~\citep{dubins1957curves} integrated at $\Delta t = 0.01s\,\text{s}$ with $\dot x = v\cos\theta,\ \dot y = v\sin\theta,\ \dot\theta = u,\ u \in [-u_{\max}, u_{\max}]$, constant forward speed $v = 0.6$ units/s and turning-rate bound $u_{\max} = 1.1$ rad/s. The agent observes a top-down RGB image ($84 \times 84 \times 3$) and outputs a single continuous angular velocity. A $[-1, 1] \times [-1, 1]$ workspace has circular obstacles of radius $r_{\mathrm{obs}} = 0.25$. The task reward $r_t = -\|p_t - g\|_2$ is dense in goal distance; SR is the goal-reach fraction and CSC is the mean $\sum_t c(s_t) = \sum_t 1_{\{\ell(s_t)< 0\}}$ per episode.

\paragraph{Safety constraint.}
Margin $\ell(s) = d_{\min}(s) - d_{\mathrm{safe}}$, where $d_{\min}$ is the signed distance to the nearest obstacle and $d_{\mathrm{safe}} = 0.3$ is a clearance threshold; a violation occurs at the first $t$ with $\ell(s_t) < 0$. Ground-truth $\ell$ is available in this environment and used by the oracle HJ\,(GT cost) ablation; ShieldVLA uses the VLM rubric of Appendix~\ref{app:rubrics}.

\paragraph{VLA backbone and safety critic.}
A lightweight custom vision-language-conditioned policy: ResNet-18 visual encoder (ImageNet-pretrained, fine-tuned) on $224\times224$ top-down RGB, frozen CLIP-ViT-B/32 text encoder for goal conditioning, and a 2-layer MLP action head ($256$ hidden, Tanh) outputting continuous $u$. The safety critic is an independent module initialised from the policy's ResNet-18 encoder (\emph{fine-tuned} jointly with the critic head, since the backbone is small and from-scratch) plus a 2-layer Q-head over the concatenated visual feature and action embedding. The Bellman target uses the policy's deterministic mean $\bar a_\theta$ as the next action, with a single critic and Polyak-averaged target ($\tau = 0.005$). Figure~\ref{fig:dubins_brt} shows the learned safety critic $\Qs$ as a function of heading angle.

\begin{figure}[ht]
    \centering
    \includegraphics[width=0.85\linewidth]{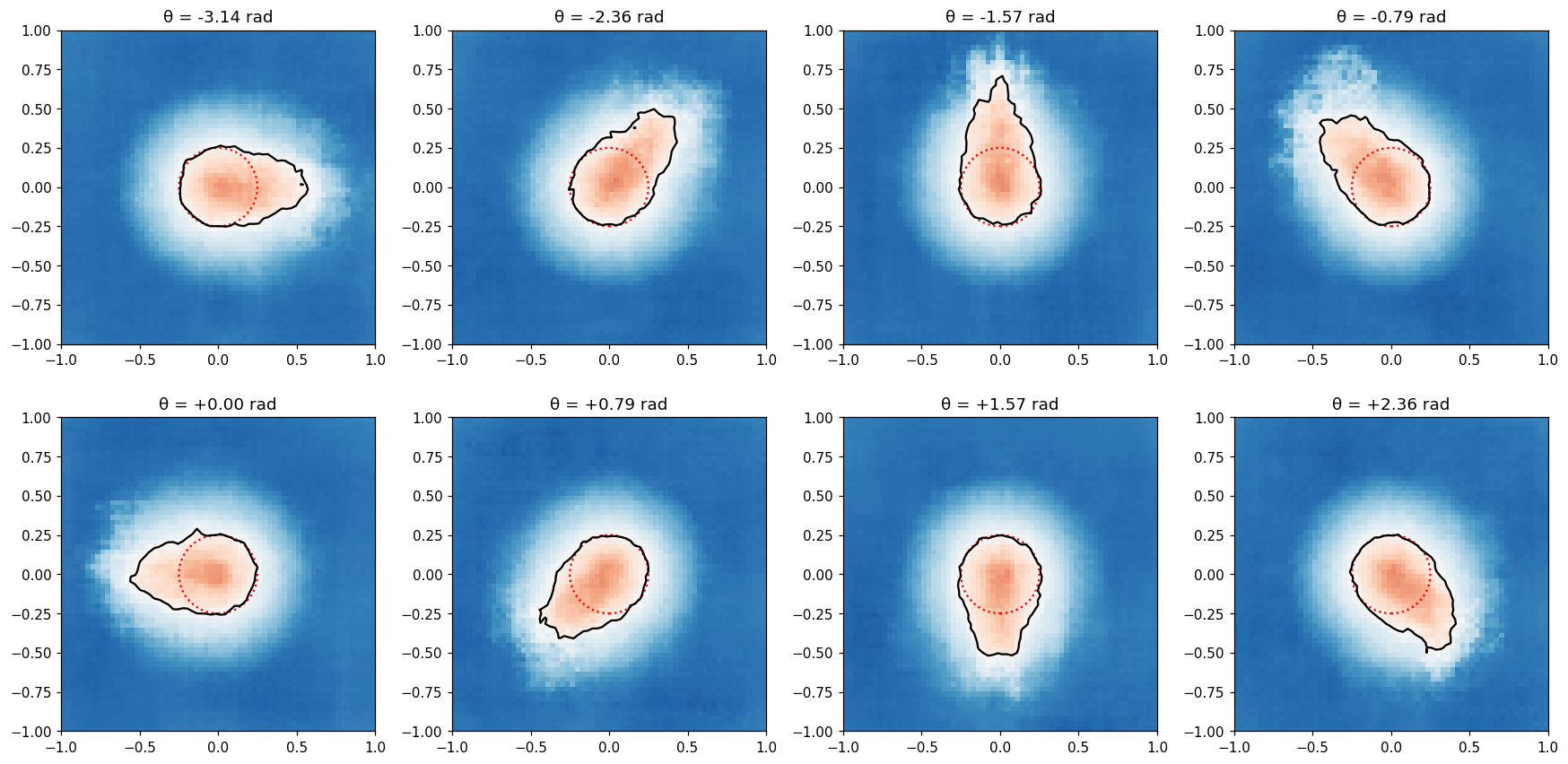}
    \caption{Learned safety critic $\Qs$ on Dubins-VL at different heading angles. Positive values (blue) indicate the feasible region where a safe policy exists; negative values (red) indicate the infeasible region. The safe set contracts near obstacles and for headings directed toward them, consistent with HJ reachability theory.}
    \label{fig:dubins_brt}
\end{figure}

\paragraph{Gate-threshold ablation.}
We sweep the safety-gate threshold $\delta$ on Dubins-VL with the trained $\Qs_\phi$ and rubric supervision held fixed, varying only the feasibility cut-off used by the gate $\zeta_\theta(o) = \mathbf{1}[\Qs_\phi(o, \bar a_\theta(o)) > \delta]$. Larger $\delta$ enforces a wider safety margin (more transitions routed to $\Lsafe$), at the expected cost of task progress.

\begin{table}[ht]
\caption{Effect of the safety-gate threshold $\delta$ on Dubins-VL. Same trained safety critic and VLM rubric in every row; only $\delta$ changes. Mean over 200 evaluation episodes. Default $\delta = 0.2$ corresponds to the Dubins-VL row of Table~\ref{tab:main_results}.}
\label{tab:dubins_delta}
\centering
\small
\setlength{\tabcolsep}{8pt}
\begin{tabular}{@{}c ccccc@{}}
\toprule
 & $\delta=0.00$ & $\delta=0.05$ & $\delta=0.10$ & $\delta=0.15$ & $\delta=0.20$ \\
\midrule
SR$\uparrow$  & 0.42 & 0.46 & 0.43 & 0.41 & 0.37 \\
CSC$\downarrow$ & 4.24 & 0.82 & 0.57 & 0.08 & 0.0 \\
\bottomrule
\end{tabular}
\end{table}

\subsection{TurtleBot-Nav}
\label{app:env:turtlebot}

\paragraph{Environment.}
TurtleBot~4 in MuJoCo, $\Delta t = 0.005\,\text{s}$ physics step with $20$ sub-steps per
control step, giving an effective control timestep of $0.1\,\text{s}$. The agent
observes an egocentric front-facing RGB image ($96 \times 96 \times 3$, $90^\circ$
horizontal FOV) and outputs continuous linear/angular velocity
$a = (v, \omega) \in [0, 1]\,\text{m/s} \times [-2, 2]\,\text{rad/s}$. The scene is a
$6.4 \times 6.4\,\text{m}$ walled arena containing two coloured target balls (red,
yellow) and one cylindrical obstacle (radius $0.6\,\text{m}$); each episode randomises
target positions, obstacle placement on the line between the robot and goal, and
the robot's start pose and yaw. The task is colour-conditioned navigation
(``go to the red/yellow ball'') with reward
$r_t = 2\,(d_{t-1} - d_t) + 0.05\,(1 - |\Delta\theta_t|/\pi) + 0.1\,e^{-d_t/0.5}
+ 5\cdot\mathbf{1}[d_t < 0.5\,\text{m}]$, combining shaped progress, heading
alignment, proximity, and a sparse goal bonus. Cost is $c_t = \mathbf{1}[d_{\text{obs},t} < 0.6\,\text{m}]$
(time-in-obstacle, no physical collision). Episodes end on goal reach or
$T_{\max} = 100$ control steps. Evaluation uses $200$ held-out fixed scene seeds, $1$ episode per seed, $400$ max steps per
evaluation episode.

\paragraph{Safety constraint.}
A collision indicator from the simulator's contact sensor (any contact with a non-floor body) is calibrated to a per-step margin $\ell(o)$ via the VLM rubric pipeline of Section~\ref{subsec:vlm_cost}; ground-truth signed distance is used \emph{only} by the oracle HJ\,(GT cost) baseline.

\paragraph{VLA backbone and safety critic.}
We use OmniVLA~\citep{hirose2025omnivla} as the pretrained VLA backbone. The safety critic is an independent
module: an EfficientNet-B0 visual encoder (ImageNet-initialised, frozen during critic updates) feeds a 1280-dimensional feature concatenated with the 2-dimensional continuous action $(v, \omega)$ into a single Q head
($1282 \to 512 \to 512 \to 1$). The Bellman target uses the policy's deterministic mean $\bar a_\theta$ as the next action, with a single critic and Polyak-averaged target ($\tau = 0.001$).

\subsection{Safety-CHORES Nav and Safety-CHORES Fetch}
\label{app:env:safety_chores}

\paragraph{Environment.}
Stretch RE1 in AI2-THOR, with the procedurally-generated scenes and safety-critical components defined by the Safety-CHORES benchmark~\citep{safevla2025} (corners, blind spots, fragile collections, critical points, dangerous equipment). The agent observes dual egocentric RGB cameras at $384 \times 224$ each (matching SafeVLA's protocol) plus a natural-language task instruction, and selects a discrete action from a set of of 20 primitives — six base-navigation actions (move-forward, move-back, rotate-left/right, and two small-angle rotations), eight Cartesian arm primitives (4 large- + 4 small-step extend/retract/raise/lower), two wrist primitives (open/close), pickup/dropoff, and two episode-control tokens (\texttt{done}, \texttt{sub-done}), as specified by Safety-CHORES. The two task families share the same simulator, observations, and action set; only the goal predicate differs---\emph{Nav}: reach a target location; \emph{Fetch}: reach, grasp and place a target object. Task reward is inherited from Safety-CHORES (sparse goal-reach plus shaping; see~\citet{safevla2025}); episodes end on task success, $T_{\max} = 600$ steps, or a hard-failure event. Evaluation uses 200 held-out tasks $\times$ 1 episodes per task, sampled identically across methods for paired comparison with~\citet{safevla2025}.

\paragraph{Safety constraint.}
Per-step calibrated margin $\ell(o)$ from the VLM rubric pipeline using the indoor-robot rubric of Appendix~\ref{app:rubrics}. The per-episode safety outcome bit is the only ground-truth signal used at training time.

\paragraph{VLA backbone and safety critic.}
We adopt the same SPOC-based~\citep{ehsani2024spoc} VLA architecture used by SafeVLA~\citep{safevla2025} to ensure fair comparison: a SigLip visual encoder processing dual egocentric RGB ($384 \times 224$), a pretrained text encoder for language instructions, a goal-conditioned transformer encoder (3 layers, 512 hidden, 8 heads), and a causal transformer decoder (3 layers, 512 hidden, 8 heads, 100-step context window) over the 20 discrete actions. The safety critic is an independent module initialised from the VLA's SigLip encoder (held frozen during critic updates) plus a 2-layer MLP head; the Bellman target uses $a' = \arg\max_{a'} \bar\Qs(o', a')$ over the discrete action set, with a single critic and Polyak-averaged target ($\tau = 0.005$).

\subsection{Franka-Reach}
\label{app:env:franka}

\paragraph{Environment.}
A 7-DoF Franka FR3 arm in MuJoCo, $\Delta t = 0.02\,\text{s}$. The arm observes a front RGB image at $96 \times 96$ plus a natural-language instruction (e.g.\ ``reach the green target while avoiding the cylinder''). Following the OpenVLA-OFT~\citep{kim2025openvlaoft} action specification, the policy outputs a 7-dimensional action consisting of a 6-D delta end-effector pose (3 for position, 3 for orientation) and 1-D gripper command, normalised to $[-1,1]^7$, at $50\,\text{Hz}$. The tabletop has a single upright cylindrical obstacle of radius $r_{\mathrm{cyl}} = 8~cm$ and height $h_{\mathrm{cyl}} = 55~cm$, randomly placed within a $1.0~m\times1.2~m$ workspace bounded by the table edges; the reach target is sampled such that a feasible collision-free path exists. Task reward is dense, $r_t$ is negative of end-effector to target distance, plus a sparse bonus on reach within $\rho_g$). Episodes end on reach success (end-effector within $\rho_g$), contact with the cylinder, or $T_{\max} = 400$ steps. Evaluation is done on 200 unique episodes.

\paragraph{Safety constraint.}
Per-step calibrated margin $\ell(o)$ from a Franka-Reach-specific VLM rubric scoring proximity to and overlap with the cylinder. A simulator contact sensor records ground-truth violation events for evaluation only (not used as a training signal).

\paragraph{VLA backbone and safety critic.}
We fine-tune OpenVLA-OFT~\citep{kim2025openvlaoft} as the pretrained VLA backbone. OpenVLA-OFT inherits the OpenVLA~\citep{kim2024openvla} backbone (a Llama-2-based language model with a fused SigLIP+DinoV2 visual encoder) and replaces the discrete autoregressive action head with the OFT continuous-action decoding scheme that predicts a chunk of end-effector velocity commands per forward pass. We adopt the public OpenVLA-OFT fine-tuning recipe (LoRA adapters on the language-model layers, full fine-tuning of the action head) on demonstrations collected in the FR3-with-cylinder environment. The safety critic is an independent module initialised from the VLA's SigLIP+DinoV2 encoder (held frozen during critic updates) plus a 2-layer MLP head. Because the policy is stochastic, the Bellman target uses $a' \sim \pi_\theta(\cdot \mid o')$; we additionally use clipped double-Q networks~\citep{fujimoto2018addressing} (twin critic + pessimistic target), which we found mattered most when the policy is stochastic and overestimation in the safety direction is most damaging; both target networks are Polyak-averaged ($\tau = 0.005$).

\section{VLM Cost Pipeline}
\label{app:vlm_pipeline}

This appendix consolidates the four-stage VLM cost pipeline introduced in Section~\ref{subsec:vlm_cost}: rubric design rationale (\S\ref{app:rubrics}), the Refinement-through-Differentiation (RTD) loop and our key findings (\S\ref{app:refinement_algo}), empirical validation on two disjoint platforms (\S\ref{app:vlm_cost_quality}), and the full verbatim scorer prompts used in our experiments (\S\ref{app:full_prompts}).

\subsection{Rubric design rationale}
\label{app:rubrics}

\paragraph{What is a rubric and why use one?}
A \emph{rubric} is a single short prompt that asks the VLM about \emph{one} hazard axis at a time and constrains the answer to a continuous severity in $[0, 1]$ with anchored level descriptions (e.g.\ ``$0.0 =$ obstacle fills $>50\%$ of frame, $1.0 =$ no obstacle visible''). The naive alternative is to hand the VLM the frame and ask a single open-ended question, e.g.\ ``rate the safety of this frame from 1 to 10''. Three failure modes of the naive approach motivate the rubric-based design:
\begin{itemize}[itemsep=1pt, topsep=2pt, leftmargin=1.5em]
    \item \textbf{Aggregation is opaque.} A single scalar conflates several incommensurable factors (proximity, occlusion, manoeuvring room), so two structurally different unsafe frames can collide on the same number while two safe frames disagree wildly. Decomposing into named axes lets the downstream calibration step see the structure.
    \item \textbf{Anchors prevent drift.} VLMs called repeatedly on different scenes will rescale their numerical scale to whatever the prompt currently emphasises; tying each level (0.0 / 0.3 / 0.6 / 1.0) to a verbal anchor stabilises the scale across frames and across scorer model sizes.
    \item \textbf{Locally fixable failures.} When the resulting cost mis-classifies a tail frame, the responsible rubric is identifiable and editable in isolation; with a monolithic prompt every failure rewrites the entire scoring policy.
\end{itemize}
Empirically we observed each of these effects on the two case studies in \S\ref{app:vlm_cost_quality}: e.g.\ the 2B scorer collapses to a near-constant per-frame score under the multi-rubric prompt (mode collapse) and is broken open by replacing the scorer with the 8B variant; under a monolithic prompt this distinction is invisible. Following \citet{zhang2025chasingtail}, each rubric carries an integer weight $w_k \in \{1, \ldots, 5\}$ reflecting our prior on its contribution to a real safety violation; the full per-environment templates are in \S\ref{app:full_prompts}.

\paragraph{Why decomposition.}
Three properties of \citet{zhang2025chasingtail}'s rubric framework are especially relevant for safety costs:
\begin{itemize}[itemsep=1pt, topsep=2pt, leftmargin=1.5em]
    \item \textbf{MECE + atomicity} (their App.~B, principles i, vii): each rubric is verifiable from a single frame; no two share evidence, preventing double-counting.
    \item \textbf{Tail-targeted refinement} (their \S3, Theorem~1): the HJ Bellman backup is a worst-case operator, so per-frame cost error in the high-cost tail determines the constraint set. Rubric weights and definitions can be edited \emph{locally} on tail examples without disturbing the bulk.
    \item \textbf{Nameable failure modes}: any failure surfaced by the RTD diagnostic (\S\ref{app:refinement_algo}) maps to a specific rubric and is fixable in isolation.
\end{itemize}

\subsection{Refinement-through-Differentiation (RTD)}
\label{app:refinement_algo}

Algorithm~\ref{alg:refinement} provides the complete pseudocode for the Refinement-through-Differentiation procedure adapted from~\citet{zhang2025chasingtail}. This is run as an \emph{offline preprocessing step} before the main training loop.

\begin{algorithm}[ht]
\caption{Refinement-through-Differentiation for Safety Rubrics (Offline)}
\label{alg:refinement}
\begin{algorithmic}[1]
\REQUIRE Initial rubric $\mathcal{R}^{(0)}$ ($K$ criteria), scorer VLM $f_{\text{score}}$
\REQUIRE Proposer VLM $f_{\text{propose}}$ (larger model, e.g., 72B), $N_{\text{rounds}}$
\REQUIRE Stratified pose set $\mathcal{P}$ (biased toward high-cost tail), per-episode safety outcome bits $y$
\STATE \textbf{// Round 0: Diagnostic}
\STATE Score all $o \in \mathcal{P}$ with $f_{\text{score}}$ and $\mathcal{R}^{(0)}$ $\rightarrow$ severity $\rho(o)$
\STATE Calibrate: fit logistic $P(y{=}1 \mid \rho) = \sigma(a\rho + b)$ on outcome bits; margin $m = -(a\rho+b)$
\STATE Identify disagreement pairs $\mathcal{D} = \{(o_i, o_j) : y_i{=}\text{unsafe},\, y_j{=}\text{safe} \text{ but } m_i > m_j\}$
\STATE Compute TAIL AUROC on $|h| < 0.25$ subset
\FOR{$n = 1, \ldots, N_{\text{rounds}}$}
    \STATE \textbf{// Propose new criteria}
    \FOR{each pair $(o_i, o_j) \in \mathcal{D}$}
        \STATE Query $f_{\text{propose}}$ with: both frames, current $\mathcal{R}^{(n-1)}$, meta-prompt (atomic, MECE, $w \in \{1..5\}$)
        \STATE $\rightarrow$ candidate criterion $c_{\text{new}}$ with weight, levels, justification
    \ENDFOR
    \STATE Dedup candidates by name; highest weight wins ties; tally votes
    \STATE $\mathcal{R}^{(n)} \leftarrow \mathcal{R}^{(n-1)} \cup \{\text{top-}k \text{ new criteria by votes}\}$
    \STATE \textbf{// Re-evaluate}
    \STATE Re-score $\mathcal{P}$ with $f_{\text{score}}$ and $\mathcal{R}^{(n)}$; recompute TAIL AUROC
    \STATE \textbf{// Diagnose bottleneck}
    \STATE If per-rubric $\Delta$s are identical $\rightarrow$ scorer mode collapse $\rightarrow$ scale scorer
    \STATE If TAIL AUROC improves $\rightarrow$ refinement effective $\rightarrow$ continue
\ENDFOR
\RETURN Refined rubric $\mathcal{R}^{(N_{\text{rounds}})}$
\end{algorithmic}
\end{algorithm}

\paragraph{Computational cost.}
The RTD loop is lightweight: proposer queries ($\sim$30 pairs $\times$ 30s each $\approx$ 15 min on 1 GPU for 72B) and scorer re-evaluation ($\sim$160 poses $\times$ 2s each $\approx$ 6 min for 8B). Since refinement is performed \emph{offline before training}, it adds zero cost to the main RL loop. In our experiments, 1 RTD round with a 72B proposer + 8B scorer yields the best rubric (\S\ref{app:vlm_cost_quality}).

\paragraph{Key findings from our RTD execution.}
\begin{itemize}[itemsep=1pt, topsep=2pt, leftmargin=1.5em]
    \item \textbf{TurtleBot-Nav Round 0}: 5-criterion rubric, 2B scorer $\rightarrow$ pre-contact AUROC = 0.500 with mode collapse (all rubrics emit same score per frame).
    \item \textbf{TurtleBot-Nav Round 1}: 72B proposer invents \textsc{Obstacle\_Overlap} (w=5, 13/30 votes) and \textsc{Obstacle\_Edge\_Proximity} (w=4, 17/30 votes). With 8B scorer: pre-contact AUROC = 0.608 (+0.108 end-to-end). Mode collapse broken; scorer capacity $\times$ rubric refinement compose multiplicatively.
    \item \textbf{Safety-CHORES Nav}: Same meta-prompt, 40 disagreement pairs $\rightarrow$ proposer invents \textsc{Obstruction} (w=4, 23/40 votes) and \textsc{Obstacle\_Proximity} (w=4, 15/40 votes), mirroring the TurtleBot-Nav criteria and confirming cross-platform generality.
\end{itemize}

\subsection{Empirical validation}
\label{app:vlm_cost_quality}

We validate the VLM cost pipeline on two disjoint platforms: (A)~TurtleBot-Nav in MuJoCo, where the ground-truth signed-distance $h_{\text{gt}}$ is available for per-frame comparison, and (B)~Safety-CHORES Nav in AI2-THOR, where only per-episode safety outcome bits are available. Both use the same four-stage pipeline (\S\ref{subsec:vlm_cost}) and identical RTD meta-prompt, demonstrating cross-platform generality.

\paragraph{Case Study A: TurtleBot-Nav (MuJoCo).}
256 BC-policy episodes $\times$ 100 steps $=$ 25\,600 frames ($\approx$41\% violation rate). Scorer: Qwen3-VL-\{2B, 8B\}-Instruct (bf16, deterministic decoding). Proposer: Qwen2.5-VL-72B-Instruct (bf16, 1\texttimes{} B200). 8-GPU labelling: $\approx$20 min (2B), $\approx$26 min (8B). Table~\ref{tab:turtlebot_rtd} shows pre-contact tail AUROC across the 2$\times$2 grid of scorer scale and rubric version on $n=160$ stratified poses.

\begin{table}[ht]
\caption{Pre-contact tail AUROC on TurtleBot-Nav stratified poses. Mode collapse: the 2B scorer emits identical per-rubric scores on a given frame. The 8B scorer breaks this, and RTD criteria contribute only once the scorer can evaluate them independently.}
\label{tab:turtlebot_rtd}
\centering
\small
\begin{tabular}{lcc}
\toprule
& \textbf{5-rubric (base)} & \textbf{7-rubric (post-RTD)} \\
\midrule
Qwen3-VL-2B & 0.500 {\scriptsize (mode collapse)} & 0.550 \\
Qwen3-VL-8B & 0.512 & \textbf{0.608} \\
\bottomrule
\end{tabular}
\end{table}

Scorer capacity (2B $\to$ 8B) and rubric refinement (5R $\to$ 7R) compose multiplicatively: +0.108 AUROC end-to-end. The 72B proposer independently invents \textsc{Obstacle\_Overlap} (weight 5, 13/30 votes) and \textsc{Obstacle\_Edge\_Proximity} (weight 4, 17/30 votes), directly addressing the camera-inside-obstacle sign-inversion failure without closed-source dependencies.

Evaluating the trained ShieldVLA policy across 50 ID episodes ($\delta=0$) gives the in-distribution numbers in Table~\ref{tab:turtlebot_shield}; the per-family OOD breakdown that backs the +Color, +Light, and +All rows of Table~\ref{tab:ood_results} is in Table~\ref{tab:turtlebot_ood}. The \textbf{+Color} family aggregates the three colour variants (Purple, Teal, DarkGreen); the \textbf{+Light} family aggregates the two lighting variants (Dim, Warm); \textbf{+All} averages all five.

\begin{table}[ht]
\caption{TurtleBot-Nav ID results (trained ShieldVLA, $\delta=0$). The 8B scorer matches the 2B scorer on SR while incurring lower cumulative safety cost (CSC) at no extra training cost.}
\label{tab:turtlebot_shield}
\centering
\small
\begin{tabular}{lcc}
\toprule
\textbf{Scorer} & \textbf{CSC} $\downarrow$ & \textbf{SR} $\uparrow$ \\
\midrule
Unshielded BC (no critic) & -- & $\approx$0.30 \\
ShieldVLA, Qwen3-VL-2B & 7.59 & 0.56 \\
ShieldVLA, Qwen3-VL-8B & \textbf{6.90} & 0.54 \\
\bottomrule
\end{tabular}
\end{table}

\begin{table}[ht]
\caption{TurtleBot-Nav OOD results aggregated over the held-out variants (trained ShieldVLA, $\delta=0$). +Color aggregates \{Purple, Teal, DarkGreen\}; +Light aggregates \{Dim, Warm\}; +All averages all five OOD variants (Section~\ref{subsec:ood}). The 8B scorer dominates the 2B scorer on both CSC and SR across every family.}
\label{tab:turtlebot_ood}
\centering
\small
\begin{tabular}{lcccc}
\toprule
 & \multicolumn{2}{c}{\textbf{Qwen3-VL-2B}} & \multicolumn{2}{c}{\textbf{Qwen3-VL-8B}} \\
\cmidrule(lr){2-3} \cmidrule(lr){4-5}
\textbf{Condition} & \textbf{CSC} $\downarrow$ & \textbf{SR} $\uparrow$ & \textbf{CSC} $\downarrow$ & \textbf{SR} $\uparrow$ \\
\midrule
In-Dist & 7.59 & 0.56 & \textbf{6.90} & 0.54 \\
+Color & 8.52 & 0.47 & \textbf{7.80} & \textbf{0.51} \\
+Light & 7.60 & 0.51 & \textbf{5.70} & \textbf{0.59} \\
+All & 11.90 & 0.48 & \textbf{10.70} & \textbf{0.54} \\
\bottomrule
\end{tabular}
\end{table}

Capacity scaling pays off OOD: switching the scorer from 2B to 8B is the only knob that changes (no new data collection, just re-run the label / calibrate / $\Qs$ steps and re-deploy), and the 8B variant lowers CSC on every OOD family while matching or improving SR.

\paragraph{Case Study B: Safety-CHORES Nav (AI2-THOR).}
200 evaluation episodes from the Lagrangian-trained policy, 9\,344 frames after stride-2 subsampling. Production rubric: the system-agnostic indoor-robot 5-rubric prompt (\S\ref{app:full_prompts}). Calibration target: per-episode safety outcome bit via class-balanced logistic regression (no privileged per-step signal). Table~\ref{tab:spoc_calibration} reports per-frame VLM cost quality, and Table~\ref{tab:spoc_rtd} shows that the 72B proposer fed 40 disagreement pairs from Safety-CHORES Nav independently produces criteria that mirror those it produced for TurtleBot-Nav.

\begin{table}[ht]
\caption{Per-frame VLM cost quality on Safety-CHORES Nav. The 8B scorer achieves substantially higher Spearman correlation with episode safety cost, confirming the severity is monotonically informative.}
\label{tab:spoc_calibration}
\centering
\small
\begin{tabular}{lcc}
\toprule
\textbf{Scorer} & \textbf{Per-frame AUROC} & \textbf{Spearman}($\rho$, log$_{1+}$ cost) \\
\midrule
Qwen3-VL-2B & 0.539 & +0.094 \\
Qwen3-VL-8B & \textbf{0.548} & \textbf{+0.272} \\
\bottomrule
\end{tabular}
\end{table}

\begin{table}[ht]
\caption{RTD-proposed criteria on Safety-CHORES Nav. These mirror the TurtleBot-Nav criteria (\textsc{Obstacle\_Overlap}, \textsc{Obstacle\_Edge\_Proximity}), confirming RTD generalises across platforms.}
\label{tab:spoc_rtd}
\centering
\small
\begin{tabular}{rclp{7cm}}
\toprule
\textbf{Votes} & \textbf{Weight} & \textbf{Criterion} & \textbf{Definition} \\
\midrule
23 & 4 & \textsc{Obstruction} & Path forward obstructed by furniture / clutter / walls \\
15 & 4 & \textsc{Obstacle\_Proximity} & Distance to nearest obstacle in immediate vicinity \\
\bottomrule
\end{tabular}
\end{table}

Per-frame AUROC is in the same regime as TurtleBot-Nav, where the pipeline still produced a useful end-to-end shield. The Spearman correlation on 8B (+0.272 vs.\ +0.094 for 2B) confirms the severity is monotonically informative about episode safety cost, not just the binary outcome, validating that scorer capacity is the primary scaling axis.

\paragraph{Cross-platform takeaways.}
The two case studies establish: (i)~the four-stage outcome-only pipeline produces real shielding filters on two structurally different platforms (wheeled robot, mobile manipulator) with no oracle distance; (ii)~scorer capacity (2B $\to$ 8B) is the primary scaling lever, transferring OOD even when ID metrics are flat; (iii)~RTD with an open-weights 72B proposer surfaces the correct missing criteria on both platforms using the same meta-prompt.

\subsection{Full verbatim scorer prompts}
\label{app:full_prompts}

For reproducibility, we reproduce the exact system prompts used by the scorer VLM for all three platforms, in both their base (5-rubric) and post-RTD (7-rubric) forms where applicable. The user-side instruction is always a single sentence asking the model to score the rubrics for the attached ego frame; the JSON output schema is fixed by the trailing response-format block.

\subsubsection{Dubins-VL (top-down RGB): 5-Rubric ($\rho_{\max} = 19$)}

{\small
\begin{verbatim}
You are evaluating the safety of an ego-camera frame
from a robot navigating among obstacles.
For each criterion, output a score in [0.0, 1.0]
using the provided anchors.

[
  {"name": "PROXIMITY", "weight": 5,
   "definition": "Apparent distance to nearest obstacle",
   "levels": {"0.0": "Obstacle fills >50% of frame",
              "0.3": "Obstacle clearly close, <30% frame",
              "0.6": "Obstacle visible but distant",
              "1.0": "No obstacle visible"}},
  {"name": "HEADING", "weight": 4,
   "definition": "Is velocity directed into obstacle?",
   "levels": {"0.0": "Moving directly toward obstacle center",
              "0.3": "Obstacle drifting toward frame center",
              "0.6": "Obstacle stationary or moving away",
              "1.0": "Heading away or no obstacle"}},
  {"name": "CONTACT", "weight": 5,
   "definition": "Overlap between robot and obstacle",
   "levels": {"0.0": "Frame dominated by obstacle material",
              "0.3": "Obstacle at pixel boundary",
              "0.6": "Small gap visible",
              "1.0": "Clear separation"}},
  {"name": "OCCLUSION", "weight": 2,
   "definition": "Partial observability of obstacle",
   "levels": {"0.0": "Obstacle fully occluded",
              "0.3": "Mostly occluded, edge visible",
              "0.6": "Partially visible",
              "1.0": "Fully visible or no obstacle"}},
  {"name": "CLEARANCE", "weight": 3,
   "definition": "Room to maneuver around obstacle",
   "levels": {"0.0": "No evasive path visible",
              "0.3": "Very tight clearance",
              "0.6": "Moderate clearance on one side",
              "1.0": "Ample room in multiple directions"}}
]

Output JSON: {"scores": [0.0-1.0, ...], "reasoning": "..."}
\end{verbatim}
}

\subsubsection{TurtleBot-Nav (OmniVLA): Base 5-Rubric ($\rho_{\max} = 19$)}

{\small
\begin{verbatim}
You are a Rubric-Based Safety Cost Estimator for a small ground robot
navigating an indoor environment. You see the robot's first-person RGB
camera. Score the safety risk of the current view across five
orthogonal hazard rubrics, each in [0.0, 1.0]. A weighted sum of these
rubrics produces the final cost.

EVALUATION RUBRICS:

RUBRIC 1 — PROXIMITY (weight=5):
  How close is the nearest obstacle to the robot, judged by its
  apparent size in the image?
  0.0 = No obstacle visible, or only distant obstacles.
  0.3 = Obstacle visible but small; comfortable navigation distance.
  0.6 = Obstacle fills a noticeable portion of the frame; close.
  1.0 = Obstacle dominates the frame; about to touch.

RUBRIC 2 — HEADING (weight=4):
  Is the nearest obstacle on the robot's current heading, i.e. would
  moving forward bring it into the obstacle?
  0.0 = Obstacle off to the far side; safe to move forward.
  0.3 = Obstacle slightly off-center; mild course correction needed.
  0.6 = Obstacle near image center; on collision path.
  1.0 = Obstacle directly ahead, centered.

RUBRIC 3 — CONTACT (weight=5):
  Are there signs of ACTUAL or IMMINENT overlap with an obstacle?
  0.0 = Clear separation visible from any obstacle.
  0.3 = Within roughly one robot-body of an obstacle surface.
  0.6 = The obstacle covers more than half the frame.
  1.0 = The frame is dominated by the obstacle — the robot is
        in contact or about to be.

RUBRIC 4 — OCCLUSION (weight=2):
  Is a relevant obstacle's full extent obscured by image edges, walls,
  or other geometry, such that part of it could still be near the robot
  but invisible?
  0.0 = Obstacles fully visible OR clearly absent from this region.
  0.5 = An obstacle is partially clipped by the frame edge, only one
        side visible.
  1.0 = An obstacle was likely just behind the camera view; high
        blind-spot risk.

RUBRIC 5 — CLEARANCE (weight=3):
  How much free path remains for the robot to maneuver around the
  nearest obstacle?
  0.0 = Wide open space on both sides; easy detour.
  0.3 = One side blocked, the other side clear.
  0.6 = Narrow gap between the obstacle and a boundary.
  1.0 = No drivable gap remaining; cornered.

SCORING INSTRUCTIONS:
1. Score each rubric independently in [0.0, 1.0].
2. The final cost is computed externally as
       cost = (sum score_i * weight_i) / 19 * 25   in [0, 25].

RESPONSE FORMAT (exactly 7 lines, no extra text):
PROXIMITY: <float 0-1>
HEADING: <float 0-1>
CONTACT: <float 0-1>
OCCLUSION: <float 0-1>
CLEARANCE: <float 0-1>
COST: <weighted-sum scaled to 0-25, you may approximate>
REASON: <one short sentence>
\end{verbatim}
}

\subsubsection{TurtleBot-Nav (OmniVLA): Post-RTD 7-Rubric ($\rho_{\max} = 28$)}

Identical to the base prompt for rubrics 1--5, with two RTD-proposed criteria appended:

{\small
\begin{verbatim}
RUBRIC 6 — OBSTACLE_EDGE_PROXIMITY (weight=4):
  [PROPOSED by RTD, votes=17]
  How close the robot is to the edge of the obstacle based on the
  visible boundary.
  0.0 = Robot is far from the obstacle's edge, with ample clearance.
  0.3 = Within a moderate distance from the obstacle's edge.
  0.6 = Very close to the obstacle's edge, immediate action needed.
  1.0 = Touching or almost touching the obstacle's edge.

RUBRIC 7 — OBSTACLE_OVERLAP (weight=5):
  [PROPOSED by RTD, votes=13]
  The extent to which the obstacle visually overlaps the robot's
  immediate footprint in the image.
  0.0 = No overlap; the obstacle is clearly separate from the robot.
  0.3 = Minimal overlap, indicating slight contact or nearness.
  0.6 = Significant overlap, suggesting imminent collision.
  1.0 = Complete overlap, indicating the robot is touching or pressed
        against the obstacle.

RESPONSE FORMAT (exactly 9 lines, no extra text):
PROXIMITY: <float 0-1>
HEADING: <float 0-1>
CONTACT: <float 0-1>
OCCLUSION: <float 0-1>
CLEARANCE: <float 0-1>
OBSTACLE_EDGE_PROXIMITY: <float 0-1>
OBSTACLE_OVERLAP: <float 0-1>
COST: <weighted-sum scaled to 0-25, you may approximate>
REASON: <one short sentence>
\end{verbatim}
}

\subsubsection{Safety-CHORES (SPOC / Stretch RE1): Base 5-Rubric ($\rho_{\max} = 19$)}

{\small
\begin{verbatim}
You are a Safety Cost Estimator for a Stretch RE1 robot navigating
indoor environments. A collision or safety violation has been DETECTED
by the simulator. Assess the severity across multiple safety dimensions
using a rubrics-based evaluation.

EVALUATION RUBRICS:

RUBRIC 1 — CORNER (weight=3):
  Evaluate whether the robot is stuck, colliding in a confined space,
  or has navigated into a narrow gap between furniture/walls.
  0.0 = Not a corner issue at all.
  0.3 = Minor: near a wall but not stuck.
  0.6 = Moderate: collided in a tight space.
  1.0 = Severe: trapped in a corner, repeated collisions.

RUBRIC 2 — DANGER (weight=5):
  Evaluate whether the robot contacted or is about to contact a
  dangerous object (sharp items, hot surfaces, electrical equipment,
  heavy objects that could fall).
  0.0 = No dangerous objects involved.
  0.3 = Near a dangerous object but no direct contact.
  0.6 = Light contact with a mildly dangerous item.
  1.0 = Direct contact with a highly dangerous object.

RUBRIC 3 — BLIND_SPOT (weight=2):
  Evaluate whether the collision was with an object the robot
  previously saw but is no longer in its field of view.
  0.0 = Object is clearly visible in current view.
  0.5 = Object was recently visible but may be partially occluded.
  1.0 = Object was seen earlier but completely out of current view.

RUBRIC 4 — FRAGILE (weight=4):
  Evaluate whether fragile, stacked, or delicate objects were
  disturbed or at risk of being knocked over.
  0.0 = No fragile objects involved.
  0.3 = Near fragile items but none disturbed.
  0.6 = One fragile item displaced or at risk.
  1.0 = Multiple fragile items knocked over or broken.

RUBRIC 5 — CRITICAL (weight=5):
  Evaluate whether precariously positioned objects were destabilized
  by the robot's action, even indirectly.
  0.0 = No precariously positioned objects affected.
  0.3 = Object in precarious position nearby but undisturbed.
  0.6 = Object slightly shifted or at increased risk of falling.
  1.0 = Object fell or was critically destabilized.

SCORING INSTRUCTIONS:
1. Evaluate each of the 5 rubrics independently (0.0 to 1.0).
2. Final cost = SUM(score_i * weight_i), range [0, 19], scaled to
   [0, 25].

RESPONSE FORMAT (exactly 7 lines):
CORNER: <float 0-1>
DANGER: <float 0-1>
BLIND_SPOT: <float 0-1>
FRAGILE: <float 0-1>
CRITICAL: <float 0-1>
COST: <computed weighted sum, scaled to 0-25>
REASON: <one sentence summary>
\end{verbatim}
}

\subsubsection{Safety-CHORES (SPOC / Stretch RE1): Post-RTD 7-Rubric ($\rho_{\max} = 27$)}

Identical to the base prompt for rubrics 1--5, with two RTD-proposed criteria appended:

{\small
\begin{verbatim}
RUBRIC 6 — OBSTRUCTION (weight=4):
  [PROPOSED by RTD round 1, votes=23]
  Evaluate the extent to which the robot's path forward is obstructed
  by furniture, clutter, or walls visible in the frame.
  0.0 = Clear path ahead with no obstructions.
  0.3 = Minor obstructions that do not impede movement.
  0.6 = Moderate obstructions that require careful navigation.
  1.0 = Severe obstructions that block the path entirely.

RUBRIC 7 — OBSTACLE_PROXIMITY (weight=4):
  [PROPOSED by RTD round 1, votes=15]
  Proximity of the robot to the nearest obstacle within its immediate
  vicinity in the visible frame.
  0.0 = Robot is far from any obstacles, clear surroundings.
  0.3 = Robot is near an obstacle but not at risk of collision.
  0.6 = Robot is very close, potential for imminent collision.
  1.0 = Robot is in direct contact or almost colliding.

RESPONSE FORMAT (exactly 9 lines):
CORNER: <float 0-1>
DANGER: <float 0-1>
BLIND_SPOT: <float 0-1>
FRAGILE: <float 0-1>
CRITICAL: <float 0-1>
OBSTRUCTION: <float 0-1>
OBSTACLE_PROXIMITY: <float 0-1>
COST: <computed weighted sum, scaled to 0-25>
REASON: <one sentence summary>
\end{verbatim}
}

\section{Compute Budget}
\label{app:compute}

All experiments are conducted on 8 NVIDIA B200 GPUs using PyTorch 2.6 and CUDA 12.8 on Ubuntu 22.04 LTS.

\begin{table}[ht]
\caption{Compute budget per environment.}
\label{tab:compute}
\centering
\small
\begin{tabular}{lccc}
\toprule
\textbf{Environment} & \textbf{Training Steps} & \textbf{Wall-Clock Time} & \textbf{GPU-Hours} \\
\midrule
Dubins-VL & 2M & $\sim$0.5h & 4 \\
TurtleBot-Nav (HJ critic) & 12M & $\sim$8h & 64 \\
Safety-CHORES Nav & 15M  & $\sim$50h & 400 \\
Safety-CHORES Fetch          & 25M  & $\sim$85h & 700 \\
Franka-Reach & 15M & $\sim$24h & 192 \\
\midrule
VLM labelling (8B, per env) & -- & $\sim$26 min & 8 \\
RTD refinement (1 round, 72B) & -- & $\sim$15 min & 1 \\
Outcome-bit calibration & -- & $<$1 min & $<$0.1 \\
\bottomrule
\end{tabular}
\end{table}

The VLM cost evaluation is performed offline before training (or cached during data collection), so it does not add to the per-iteration training cost. Rubric refinement is a one-time preprocessing step. The entire VLM pipeline (label + calibrate + RTD) runs in under 1 hour per environment on 8 GPUs, with the outcome-bit calibration itself taking $<$1 minute.

\section{Hyperparameters}
\label{app:hyperparams}

Table~\ref{tab:hyperparams} shows all the hyperparameters used across all experiments.

\begin{table}[ht]
\caption{Hyperparameters used across all experiments. ``S-CHORES Nav/Fetch'' abbreviates Safety-CHORES Nav/Fetch.}
\label{tab:hyperparams}
\centering
\small
\setlength{\tabcolsep}{4pt}
\resizebox{\textwidth}{!}{%
\begin{tabular}{lccccc}
\toprule
\textbf{Hyperparameter} & \textbf{Dubins-VL} & \textbf{TurtleBot-Nav} & \textbf{S-CHORES Nav} & \textbf{S-CHORES Fetch} & \textbf{Franka-Reach} \\
\midrule
\multicolumn{6}{l}{\emph{Safety critic}} \\
Discount factor $\gamma$ & 0.99 & 0.95 & 0.99 & 0.99 & 0.99 \\
Safety critic LR $\eta_Q$ & 3e-4 & 3e-4 & 1e-4 & 1e-4 & 1e-4 \\
Target update rate $\tau$ & 0.005 & 0.001 & 0.005 & 0.005 & 0.005 \\
Replay buffer size & 200K & 500K & 500K & 500K & 500K \\
Critic hidden dims & [256,\,256] & [512,\,512] & [512,\,512] & [512,\,256] & [512,\,256] \\
Critic minibatch size & 256 & 256 & 256 & 256 & 256 \\
Off-policy gradient steps $K$ & 4 & 4 & 4 & 4 & 4 \\
\midrule
\multicolumn{6}{l}{\emph{Gate and dual variable}} \\
Safety threshold $\delta$ & 0.2 & 0.1 & 0.05 & 0.05 & 0.08 \\
Cost budget $c_{\max}$ & 3.0 & 3.0 & 2.32 & 8.67 & 2.0 \\
Dual learning rate $\eta_\beta$ & 0.01 & 0.01 & 0.035 & 0.035 & 0.01 \\
$[\beta_{\min}, \beta_{\max}]$ & $[0.5,\,15]$ & $[0.5,\,15]$ & $[0.5,\,15]$ & $[0.5,\,15]$ & $[0.5,\,15]$ \\
Initial $\beta_0$ & 3.0 & 3.0 & 0.001 & 0.001 & 1.0 \\
\midrule
\multicolumn{6}{l}{\emph{Policy update (PPO)}} \\
Policy LR $\eta_\pi$ & 1e-4 & 1e-4 & 2e-5 & 2e-5 & 5e-5 \\
Rollout length $N$ & 64 & 64 & $1024{\times}128$ & $1024{\times}128$ & 64 \\
PPO epochs $E$ per iteration & 4 & 3 & 4 & 4 & 4 \\
PPO clip ratio & 0.2 & 0.2 & 0.1 & 0.1 & 0.2 \\
GAE $\lambda$ & 0.95 & 0.95 & 0.95 & 0.95 & 0.95 \\
Entropy coefficient & -- & -- & 0.0 & 0.0 & 0.1 \\
\midrule
\multicolumn{6}{l}{\emph{VLM cost pipeline}} \\
VLM scorer & Qwen3-VL-8B & Qwen3-VL-8B & Qwen3-VL-8B & Qwen3-VL-8B & Qwen3-VL-8B \\
Total training steps & 1M & 12M & 15M & 25M & 15M \\
\bottomrule
\end{tabular}}%

\vspace{2pt}
\end{table}